\documentclass[11pt]{article}

\usepackage[utf8]{inputenc} 
\usepackage[T1]{fontenc}    
\usepackage{times}          
\usepackage[margin=1in]{geometry}
\usepackage{hyperref}       
\usepackage{url}            
\usepackage{booktabs}       
\usepackage{amsfonts}       
\usepackage{nicefrac}       
\usepackage{microtype}      
\usepackage{xcolor}         

\usepackage{amsmath}
\usepackage{amssymb}
\usepackage{mathtools}
\usepackage{amsthm}
\usepackage{multirow}
\usepackage{threeparttable}
\usepackage{enumitem}

\usepackage[capitalize,noabbrev]{cleveref}
\usepackage{wrapfig}
\usepackage{graphicx}
\usepackage{subcaption}
\usepackage{algorithmic}
\usepackage{algorithm}
\usepackage{float}
\usepackage{natbib}
\theoremstyle{plain}
\newtheorem{theorem}{Theorem}[section]
\newtheorem{proposition}[theorem]{Proposition}
\newtheorem{lemma}[theorem]{Lemma}
\newtheorem{corollary}[theorem]{Corollary}
\theoremstyle{definition}
\newtheorem{definition}[theorem]{Definition}

\theoremstyle{remark}
\newtheorem{remark}[theorem]{Remark}

\definecolor{MyColor0}{RGB}{146,136,134}
\definecolor{MyColor1}{RGB}{0,0,0}
\definecolor{MyColor2}{RGB}{0,0,0}

\newcommand{\tp}[0]{{\tilde \pi}}
\newcommand{\1}[0]{\mathbb{I}}

\newcommand{\argmaxl}[0]{\operatorname*{argmax}}

\title{ANO: Robust Policy Optimization via Bounded, Redescending Gain Fields}

\author{%
Yiheng Zhang$^{1}$, Yiming Wang$^{1}$, Kaiyan Zhao$^{2}$\\
Zhenglin Wan$^{3}$, Jiayu Chen$^{4}$, Leong Hou U$^{1}$\\[6pt]
$^{1}$University of Macau \qquad $^{2}$Wuhan University\\
$^{3}$National University of Singapore \qquad $^{4}$University of Hong Kong\\[4pt]
\texttt{ryanlhu@um.edu.mo}
}
\date{}

\begin{document}

\maketitle

\begin{abstract}
    Proximal Policy Optimization (PPO) dominates reinforcement learning and LLM alignment, yet its hard-clipping mechanism and unconstrained alternatives (e.g., SPO) sit at two extremes of a stability-efficiency dilemma. We argue that this dilemma is best understood \emph{dynamically}: a surrogate objective is a feedback law on the probability ratio, and its clipping/penalty shape defines a \emph{gain field} that drives the update dynamics. PPO's clip induces a \emph{dead zone} (zero feedback outside the trust region), leaving the policy to drift open-loop under momentum; SPO's quadratic penalty induces an \emph{unbounded, linearly growing gain} that stiffens the dynamics and destabilizes under aggressive step sizes. Guided by this view, we derive \textbf{Anchored Neighborhood Optimization (ANO)}, which designs the gain field directly: a $C^\infty$ shaping kernel that anchors the identity map at $r{=}1$, peaks exactly at a prescribed trust-region boundary $1{+}\epsilon$, bounds the push on severely off-policy samples by a tunable $\kappa_{+}$, and exerts a bounded, \emph{redescending} \emph{pull} of tunable depth $\kappa_{-}$ on extreme outliers. The three hyperparameters have decoupled roles, and all internal constants are solved in \emph{closed form}. Empirically, ANO ranks \emph{first on both} Atari (40 games) and MuJoCo in IQM and Median of normalized scores. While the runner-up differs across domains (PAPO on Atari, SPO on MuJoCo), ANO is the only method consistently at the top. Under a learning-rate stress test ($3\times10^{-4}\!\to\!10^{-3}$), ANO degrades by only $0.9\%$ whereas PPO collapses by $54.5\%$, and the stressed ANO still outperforms PPO and PAPO at their best-tuned learning rates. Our code is available at \href{https://anonymous.4open.science/r/ano-F818}{https://anonymous.4open.science/r/ano-F818}.
\end{abstract}

\section{Introduction}

Deep Reinforcement Learning (DRL) has achieved remarkable success in areas such as gaming, robotic control, and language model alignment \citep{mnih2015human, silver2016mastering, silver2018general, vinyals2019grandmaster, ye2020mastering, makoviychuk2021isaac, rudin2022learning, heess2017emergence, schulman2015high, ouyang2022training, black2023training}. However, Policy Gradient (PG) methods face a fundamental trade-off between update stability and learning efficiency. Trust Region Policy Optimization (TRPO) enforces stability through second-order constraints but is computationally expensive \citep{schulman2015trust}. Proximal Policy Optimization (PPO) addressed this by introducing a first-order clipping mechanism, achieving wide adoption due to its simplicity \citep{schulman2017proximal}.

Despite its popularity, PPO's ``hard clipping'' frequently fails to enforce the true trust region \citep{ilyas2018deep, wang2020truly} and depends on code-level optimizations for stability \citep{engstrom2020implementation}. We argue that these pathologies are most clearly understood, and most directly fixed, through a \textbf{dynamical lens}. A clipped or penalized surrogate objective induces, for every sample, an effective feedback law on the probability ratio $r(s,a)=\tp(a|s)/\pi(a|s)$: the derivative of the shaping function acts as a \emph{gain field} that accelerates, brakes, or reverses the movement of $r$. From this viewpoint, existing methods instantiate two extreme and equally undesirable feedback laws:
\begin{itemize}[leftmargin=*, itemsep=0mm, topsep=1pt]
    \item \textbf{PPO: a dead zone.} Beyond the clipping interval the gain is exactly zero: once a sample's ratio crosses $1+\epsilon$, the objective provides \emph{no feedback at all}, and under momentum-based optimizers the ratio keeps drifting open-loop while the gradient information carried by outlier samples is discarded.
    \item \textbf{SPO: unbounded linear gain.} Quadratic penalties \citep{xie2025simple} restore feedback, but with a gain that grows \emph{linearly} in the violation: for extreme outliers a single sample can dominate the batch gradient (interference), and on the probability simplex forcefully suppressing one action unpredictably redistributes mass onto the others.
\end{itemize}

This perspective suggests that the remedy is not a better constraint enforcement mechanism, but a better \emph{feedback law}: the gain field should (i) coincide with the identity at $r=1$ to preserve the on-policy gradient, (ii) turn negative beyond a prescribed trust-region boundary so that violations are \emph{actively} corrected, (iii) remain \emph{bounded} everywhere so that no single sample can dominate a batch update, and (iv) \emph{redescend} to zero for extreme outliers, whose information content is unreliable. Requirements (i)--(iv) echo redescending M-estimators in robust statistics \citep{huber2011robust, huber1992robust}, where bounded influence and outlier rejection are the classical design goals.

Guided by these principles, we propose \textbf{Anchored Neighborhood Optimization (ANO)}. Rather than tuning a fixed accidental shape, ANO constructs the shaping function \emph{from its desired gain profile}: we parameterize the derivative directly as a rational function of an exponential, impose the anchors $G(1)=1$, $G'(1)=1$, $G'(1+\epsilon)=0$, and prescribe the two bounds of the gain field: a push bound $\kappa_{+}$ (saturated as $r \to -\infty$) and a pull floor $\kappa_{-}$ as the global minimum of the gain. The three hyperparameters $(\epsilon, \kappa_{+}, \kappa_{-})$ have decoupled roles, and every internal constant of the kernel is recovered in \emph{closed form}: no root-finding, no iterative solve, and no feasibility gap within the admissible parameter domain. The resulting kernel is $C^\infty$, has exactly one inflection point in its gain profile by construction, and evaluates in $O(1)$ per sample.

Empirically, we evaluate ANO on the full Atari benchmark (40 games) and 6 MuJoCo continuous-control tasks with 5 seeds, reporting \textit{rliable} aggregate metrics \citep{agarwal2021deep}. ANO ranks \emph{first on both domains} in both IQM and Median of normalized scores. Notably, the runner-up is domain-dependent (PAPO on Atari, SPO on MuJoCo), making ANO the \emph{only} method consistently at the top. A learning-rate stress test further confirms the dynamical account (Section~\ref{sec:robustness}).

We summarize our main contributions as follows:
\begin{itemize}[leftmargin=*, itemsep=-0mm, topsep=0pt]
    \item \textbf{A dynamical perspective on surrogate design.} We reinterpret clipped/penalized policy objectives as feedback laws on the probability ratio, identify the dead-zone (PPO) and unbounded-gain (SPO) pathologies, and distill the necessary ingredients of a robust gain field, separating them from unnecessary machinery such as second-order enforcement and hard constraints.
    \item \textbf{A principled, closed-form kernel.} We derive ANO's shaping kernel from a prescribed gain profile with three interpretable hyperparameters with decoupled roles $(\epsilon, \kappa_{+}, \kappa_{-})$, and prove its structural properties (unique inflection point, bijective depth control, unconditional solvability) that hold uniformly over the whole parameter domain.
    \item \textbf{Cross-domain empirical consistency and structural robustness.} ANO ranks first on both Atari and MuJoCo in IQM and Median, and exhibits an order-of-magnitude smaller degradation than PPO under aggressive learning rates, corroborating the dynamical stability analysis. 
\end{itemize}

\section{Preliminaries}
\label{sec:preliminaries}

We consider a Markov Decision Process with expected discounted return $\eta(\pi)=\mathbb{E}_{\tau\sim\pi}[\sum_{t=0}^{\infty}\gamma^tr_t]$, advantage function $A_{\pi}(s,a)=Q_{\pi}(s,a)-V_{\pi}(s)$, and importance ratio $r(s,a) = \frac{\tilde{\pi}(a|s)}{\pi(a|s)}$ between a candidate policy $\tilde\pi$ and the current policy $\pi$. Policy search methods optimize a local surrogate obtained from the performance difference lemma~\citep{kakade2002approximately,schulman2015trust}:
\begin{equation}\label{eq:surrogate_obj}
    S(\tilde{\pi}) = \eta(\pi) + \frac{1}{1-\gamma}\mathbb{E}_{s\sim\rho_{\pi}, a\sim\pi} \left[ r(s,a) A_{\pi}(s,a) \right],
\end{equation}
which lower-bounds the true return $\eta(\tilde{\pi})$ up to a divergence penalty; \citet{schulman2015trust} use $\beta D_{\mathrm{KL}}^{\mathrm{max}}(\pi,\tilde{\pi})$ via Pinsker's inequality and enforce the restriction $\mathbb{E}_{s}\left[D_{\mathrm{KL}}(\pi \Vert \tilde{\pi})\right] \le \epsilon$ (TRPO). Proximal Policy Optimization (PPO)~\citep{schulman2017proximal} replaces the constraint with a first-order approximation that clips the ratio:
\begin{equation}\label{eq:ppo_clip}
    \mathcal{L}^{\mathrm{PPO}}(\tilde{\pi}) = \mathbb{E}_{\substack{s\sim{\color{MyColor1}\rho_{\pi}(\cdot)} \\ a\sim{\color{MyColor1}\pi(\cdot|s)}}} \left[ \min\left( r(s,a) A_{\pi}, \text{clip}(r(s,a), 1-\epsilon, 1+\epsilon) A_{\pi} \right) \right].
\end{equation}

In Section~\ref{sec:design_space} we reinterpret objectives of this form as feedback laws on the probability ratio, and show that the clipping shape, rather than the constraint it proxies, is what determines the stability of the induced update dynamics.

\section{What Matters in Surrogate Design: A Dynamical Perspective}
\label{sec:design_space}

Before deriving our algorithm, we revisit the design of surrogate objectives and ask: \emph{which ingredients of PPO-style objectives are truly necessary, and which are accidental?} First, we show that two-sided, element-wise control of the probability ratio is the \emph{necessary} safety ingredient, grounded in a worst-case lower bound on the true performance (Section~\ref{sec:twosided}). Second, we argue that the \emph{shape} of the ratio mapping, equivalently the gain field it induces on the update dynamics, is what separates stable methods from brittle ones (Section~\ref{sec:gainfield}).

\subsection{Two-Sided Ratio Control: The Necessary Safety Ingredient}
\label{sec:twosided}

We begin with a worst-case lower bound that motivates controlling the probability ratio \emph{element-wise and on both sides}.

\begin{theorem}[The Dual-Ratio Lower Bound]
\label{thm:dual_bound}
Let $\eta(\tp)=\mathbb{E}_{\tau\sim\tp}[\sum_{t=0}^{\infty}\gamma^tr_t]$ denote the expected return of policy $\tp$, and let $S(\tp)$ be the surrogate objective defined in Eq.~\ref{eq:surrogate_obj}. The following bound holds:
\begin{align}
\label{eq:dual_bound}
    \eta(\tp)\ge  S(\tp) - \frac{\beta\alpha}{2}\max_{s,a}\left[\ln\frac{\tp(a|s)}{\pi(a|s)}\right]\nonumber -\frac{\beta(1-\alpha)}{2}\max_{s,a}\left[\ln\frac{\pi(a|s)}{\tp(a|s)}\right],
\end{align}
where $\alpha\in[0,1]$ is a hyperparameter, and equality holds when $\tp=\pi$.
\begin{proof}
See Appendix~\ref{app:proof_thm_3_1}.
\end{proof}
\end{theorem}

Theorem~\ref{thm:dual_bound} tightens the classical TV-based bound of \citet{schulman2015trust} into a statement about \emph{worst-case log-ratios in both directions}: the performance penalty is governed by the single most displaced sample, not by an average divergence, so safety requires bounding $r(s,a)$ element-wise on both sides (forward and reverse ratio). This motivates an objective family that enforces two-sided control geometrically.

\begin{definition}[The Generalized Surrogate Objective]
\label{def:F}
Let $r(s,a) = \frac{\tp(a|s)}{\pi(a|s)}$. Let $f, g: \mathbb{R} \to \overline{\mathbb{R}}$ be extended real-valued shaping functions satisfying two \textbf{geometric constraints}:
\begin{enumerate}[leftmargin=*, itemsep=0.5mm, topsep=0pt]
    \item \textbf{Identity Anchoring (Fixed Point Condition):} $f(1) = g(1) = 1$.
    \item \textbf{Trust Region Bounding (Geometric Enclosure):} $g(x) \ge x \ge f(x)$ for all $x \in \mathbb{R}$.
\end{enumerate}
The \textbf{Generalized Surrogate Objective} is
\begin{equation}
M(\tp;f,g) =
    \mathbb{E}_{\substack{s\sim{\color{MyColor1}\rho_{\pi}(\cdot)} \\ a\sim{\color{MyColor1}\pi(\cdot|s)}}}
    \Big[ \min \big( g(r) A_{\pi}(s,a), \, f(r) A_{\pi}(s,a) \big) \Big].
\end{equation}
When $g(r) = 2 - f(2-r)$ (point symmetry about $(1,1)$), reinforcement and suppression share one shape; we term this the \textbf{Symmetric Objective} $M_s(\tp;f)$, the form used throughout this paper.
\end{definition}

\begin{corollary}
\label{cor:F0}
For the current policy $\tp=\pi$, we have $M(\pi;f,g)=0$.
\end{corollary}
\begin{proof}
When $\tp=\pi$, $r\equiv 1$ and $\min(g(1)A, f(1)A) = A$, so $M(\pi;f,g) = \mathbb{E}_{s\sim\rho_\pi, a\sim\pi}[A_\pi(s,a)] = 0$ by definition of the advantage function.
\end{proof}

The safety value of this construction is certified by hard-constraining the shaped objective to a worst-case ratio interval, which provably yields monotonic improvement:

\begin{theorem}[The Safe Policy Update Condition]
\label{thm:base_algo}
Let $\delta_C(\cdot)$ be the extended real-valued indicator function ($0$ if $x \in C$, $\infty$ otherwise). For any shaping functions $f$ and $g$ satisfying the geometric constraints in Definition~\ref{def:F}, there exist non-negative constants $\epsilon_l, \epsilon_u$ defining a feasible interval $I=[1-\epsilon_l, 1+\epsilon_u]$. If $\pi^*$ maximizes the constrained objective
\begin{equation}
    M(\pi^*; f_\delta, g_\delta) = \max_{\tp} M(\tp; f_\delta, g_\delta),
\end{equation}
where $f_\delta = f - \delta_I$ and $g_\delta = g + \delta_I$, then monotonic improvement is guaranteed, i.e., $\eta(\pi^*) \ge \eta(\pi)$.
\begin{proof}
See Appendix~\ref{app:proof_thm_3_4}, which builds on Theorem~\ref{thm:dual_bound}.
\end{proof}
\end{theorem}

\paragraph{From hard constraints to shaped feedback.} Directly enforcing the interval constraint of Theorem~\ref{thm:base_algo} at every step is computationally intractable and yields prohibitively small updates \citep{schulman2015trust}. Following the established relaxation paradigm (TRPO's local approximation, PPO's clipping, SPO's quadratic penalty), we treat the boundary as a \emph{tunable} $\epsilon$ and replace the hard indicator with a smooth shape. The design question reduces to: \emph{what should the shape of $f$ be?} Section~\ref{sec:gainfield} answers from the dynamical viewpoint.

\paragraph{PPO and SPO as instances.} Within the symmetric objective, PPO-Clip corresponds to softening the hard indicator into a saturation on a linear base, $f_{\text{PPO}}(r) = r - \1_{r>1+\epsilon}(r-(1+\epsilon)) = \min(r, 1+\epsilon)$:
\begin{equation}
\label{eq:f_ppo}
    f_{\text{PPO}}(r) = \min(r, 1+\epsilon),
\end{equation}
and SPO corresponds to the unique quadratic base function with vertex at $1+\epsilon$ satisfying the geometric constraints ($f(1)=1$, $f(r)\le r$):
\begin{equation}
\label{eq:f_spo}
    f_{\text{SPO}}(r) = -\frac{1}{2\epsilon}(r-1-\epsilon)^2 + \frac{\epsilon}{2} + 1.
\end{equation}
Its induced gain $\kappa(r) = (1+\epsilon-r)/\epsilon$ coincides with the hinge gain of SPO~\citep{xie2025simple} wherever the latter is active, and grows linearly without bound; both are points in a common design space, and their failure modes can be compared on equal footing.

\subsection{The Gain Field: How Shaping Functions Drive the Update Dynamics}
\label{sec:gainfield}

For a single sample with advantage $A$, the shaped term satisfies $\nabla_\theta \big[f(r)\,A\big] = A\, f'(r)\, \nabla_\theta r$: the derivative $f'(r)$ acts as a state-dependent \emph{gain} on the vanilla policy-gradient direction. Isolating this effect in a scalar idealization, the ratio evolves under the velocity field
\begin{equation}
\label{eq:velocity}
    \dot r \;\propto\; A \cdot \kappa(r), \qquad \kappa(r) := f'(r),
\end{equation}
which we call the \textbf{gain field}. Eq.~\ref{eq:velocity} exposes the feedback law that each method implicitly implements (Figure~\ref{fig:kernel_geometry}b--c):

\begin{itemize}[leftmargin=*, itemsep=0.5mm, topsep=1pt]
    \item \textbf{PPO: a dead zone.} $\kappa(r) = 1$ inside $[0, 1+\epsilon]$, then $0$: beyond the boundary the system runs \emph{open-loop}: no restoring force, drift under momentum, and outlier information discarded. The trust region is enforced in the objective's \emph{value}, not its \emph{dynamics}.
    \item \textbf{SPO: unbounded linear gain.} $\kappa(r) = (1+\epsilon-r)/\epsilon$ grows with the violation: the restoring force is \emph{stiff}, dominates the batch for extreme outliers (gradient interference), redistributes probability mass in uncontrolled ways on the simplex, and overshoots under aggressive step sizes.
    \item \textbf{What is needed.} A \emph{bounded, redescending restoring force}: unit gain at $r=1$, sign reversal at $1+\epsilon$ (a stable fixed point of the velocity field), bounded magnitude everywhere (no sample dominates an update), and gain decaying to zero as $r \to +\infty$ (extreme outliers are too unreliable to act upon; the redescending principle of robust M-estimators \citep{hampel1986robust,huber2011robust}). We use this scalar flow as a design lens; it does not by itself imply parameter-space convergence; adaptive optimizers further rescale the effective gain.
\end{itemize}

This analysis separates the necessary from the accidental in PPO-style design. \textbf{Necessary}: (i) identity anchoring with unit gain, which preserves the on-policy gradient; (ii) the two-sided enclosure of Definition~\ref{def:F} ($g \ge r \ge f$), which Theorem~\ref{thm:base_algo} ties to monotonic improvement; (iii) a bounded gain field; (iv) smoothness. \textbf{Parsimonious}: the point symmetry $g(r) = 2 - f(2-r)$, which halves the design space; asymmetric extensions are promising future work (Section~\ref{sec:conclusion}). \textbf{Accidental}: hard clipping plateaus, second-order enforcement, and the exact bound coefficients. Section~\ref{sec:ano} constructs a kernel with exactly the necessary ingredients; Section~\ref{sec:experiments} verifies the predicted robustness.

\section{Anchored Neighborhood Optimization}
\label{sec:ano}

We now construct \textbf{Anchored Neighborhood Optimization (ANO)}: a shaping kernel whose gain field realizes exactly the feedback law identified in Section~\ref{sec:gainfield}. Rather than perturbing an existing clip, we design the kernel \emph{in the derivative domain} and integrate. See Appendix~\ref{app:sec:pseudocode} for pseudocode.

\subsection{Design Requirements}
\label{sec:ano_requirements}

We seek a globally $C^\infty$, non-piecewise shaping function $G: \mathbb{R} \to \mathbb{R}$ whose derivative realizes a prescribed ``push--pull'' gain profile: the gain field $\kappa = G'$ is enclosed in $[\kappa_{-}, \kappa_{+})$. Given hyperparameters $(\epsilon, \kappa_{+}, \kappa_{-})$ with
\begin{equation}
\label{eq:param_domain}
    \epsilon > 0, \qquad \kappa_{+} > 1, \qquad -\kappa_{+} < \kappa_{-} < 0,
\end{equation}
the requirements are:
\begin{enumerate}[leftmargin=*, itemsep=0.5mm, topsep=1pt, label=\textbf{(R\arabic*)}]
    \item \textbf{Anchoring:} $G(1)=1$ and $G'(1)=1$
    \item \textbf{Sign reversal:} $G'(1+\epsilon)=0$
    \item \textbf{Asymptotics:} $\lim_{x\to-\infty} G'(x) = \kappa_{+}$ and $\lim_{x\to+\infty} G'(x) = 0$
    \item \textbf{Extremum:} $\min_{x\in\mathbb{R}} G'(x) = \kappa_{-}$
    \item \textbf{Parsimony:} $G''$ vanishes and changes sign at exactly one point $x_p$
\end{enumerate}
The three knobs have direct dynamical meanings (Table~\ref{tab:knobs}): $\epsilon$ places the trust-region boundary, $\kappa_{+}$ caps the \emph{push} on severely off-policy samples (saturating as $r \to -\infty$), and $\kappa_{-}$ caps the \emph{pull} on outliers. The left tail is not an artifact: the symmetric dual $g(r) = 2 - G(2-r)$ routes overshooting samples ($r > 2$, negative advantages) through it, and its asymptote $\kappa_{+}$ bounds that restoring force.

\begin{table}[t]
\centering
\caption{\textbf{The three knobs of ANO.} Each hyperparameter controls one interpretable feature of the gain field $G'(r)$; all internal constants are then recovered in closed form (Eq.~\ref{eq:closed_form_chain}).}
\label{tab:knobs}
\begin{tabular}{lll}
\toprule
\textbf{Knob} & \textbf{Definition} & \textbf{Dynamical role} \\
\midrule
$\epsilon$ & $G'(1+\epsilon)=0$ & trust-region boundary (location of the stable fixed point) \\
$\kappa_{+}$ & $\lim_{r\to-\infty} G'(r)$ & caps the push on severely off-policy samples \\
$\kappa_{-}$ & $\min_r G'(r)$ & maximal pull on outliers (bounded restoring force) \\
\bottomrule
\end{tabular}
\end{table}

\subsection{The Gain Profile and Its Closed-Form Integration}
\label{sec:ano_construction}

Let $x_0 = 1+\epsilon$ and $E(x) = e^{a(x-x_0)}$ with scale $a>0$. We parameterize the gain field directly as a rational function of $E$, with an internal shape parameter $\nu > 0$:
\begin{equation}
\label{eq:ano_gain}
    G'(x) = \frac{\kappa_{+}\,\nu\,\big(1 - E(x)\big)}{\big(1 + E(x)\big)\,\big(E(x) + \nu\big)}.
\end{equation}
Eq.~\ref{eq:ano_gain} is the rational form of a difference of two logistic sigmoids sharing slope $a$ but centered at different points; since $E > 0$ and $\nu > 0$, it is manifestly $C^\infty$ on $\mathbb{R}$. The choice of this family is not arbitrary: three of the five requirements hold \emph{identically}, for every $(a, \nu)$:
\begin{equation}
\label{eq:identical_reqs}
    \underbrace{E(x_0) = 1 \Rightarrow G'(x_0) = 0}_{\text{(R2)}}, \qquad
    \underbrace{E \to 0 \Rightarrow G' \to \kappa_{+}}_{\text{(R3) left}}, \qquad
    \underbrace{E \to \infty \Rightarrow G' \sim -\kappa_{+}\nu/E \to 0}_{\text{(R3) right}}.
\end{equation}
No equation needs solving to enforce them, so (R2)--(R3) impose no feasibility restriction on $(\epsilon, \kappa_{+}, \kappa_{-})$; the remaining requirements are discharged by the following structural results.

\begin{proposition}[Structural uniqueness; discharges (R5)]
\label{prop:unique_inflection}
For every $a > 0$, $\nu > 0$, $\kappa_{+} > 0$, the second derivative
\begin{equation}
\label{eq:ano_gpp}
    G''(x) = \frac{\kappa_{+} \nu a\, E \left(E^2 - 2E - (1+2\nu)\right)}{\big[(1+E)(E+\nu)\big]^2}
\end{equation}
vanishes at exactly one point $x_p \in \mathbb{R}$, changes sign there, and $G'$ attains its global minimum at $x_p$:
\begin{equation}
    x_p = x_0 + \frac{\ln(1+s)}{a}, \qquad s := \sqrt{2 + 2\nu}.
\end{equation}
\end{proposition}

\begin{lemma}[Depth control; discharges (R4)]
\label{lem:depth}
At the inflection point, the minimal gain collapses to $G'(x_p) = -\kappa_{+}\, d(\nu)$ with the \emph{normalized depth}
\begin{equation}
\label{eq:depth}
    d(\nu) = \frac{s^2 - 2}{(s+2)^2}, \qquad s = \sqrt{2+2\nu},
\end{equation}
which depends on $\nu$ alone, not on $\epsilon$ or $\kappa_{+}$. Moreover, $d: (0,\infty) \to (0,1)$ is a strictly increasing $C^\infty$ bijection, so $\kappa_{-} = -\kappa_{+} d(\nu)$ sweeps $(-\kappa_{+}, 0)$ bijectively, and the parameterization inverts in closed form:
\begin{equation}
\label{eq:depth_inversion}
    d = \frac{|\kappa_{-}|}{\kappa_{+}}, \qquad
    s = \frac{2d + \sqrt{2(d+1)}}{1-d}, \qquad
    \nu = \frac{s^2-2}{2}.
\end{equation}
The restriction $|\kappa_{-}| < \kappa_{+}$ in Eq.~\ref{eq:param_domain} is intrinsic to the family~\ref{eq:ano_gain}: the pull-to-push ratio satisfies $|\kappa_{-}|/\kappa_{+} = d(\nu) < 1$ for all $\nu > 0$.
\end{lemma}

\begin{proposition}[Scale calibration; discharges (R1)]
\label{prop:scale}
Write $q := E(1) = e^{-a\epsilon} \in (0,1)$. The anchor condition $G'(1) = 1$ reads
\begin{equation}
\label{eq:scale_condition}
    \frac{\kappa_{+} \nu (1-q)}{(1+q)(q+\nu)} = 1
    \;\Longleftrightarrow\;
    q^2 + B q - C = 0, \quad
    B := \nu(\kappa_{+}+1)+1 > 0, \;\; C := \nu(\kappa_{+}-1) > 0,
\end{equation}
which has exactly one admissible root, giving the scale in closed form:
\begin{equation}
\label{eq:scale_solution}
    q = \frac{2C}{B + \sqrt{B^2 + 4C}} \in (0,1),
    \qquad
    a = -\frac{\ln q}{\epsilon} > 0.
\end{equation}
The conjugate form in Eq.~\ref{eq:scale_solution} avoids catastrophic cancellation when $C \ll B^2$. For every $(\epsilon, \kappa_{+}, \kappa_{-})$ in the domain~\ref{eq:param_domain}, a unique $a$ exists (unconditional solvability).
\end{proposition}
\begin{lemma}[Two-sided enclosure]
\label{lem:enclosure}
For every feasible $(\epsilon, \kappa_{+}, \kappa_{-})$, $G(x) \le x$ for all $x \in \mathbb{R}$; hence the Symmetric Objective satisfies the enclosure $g(r) \ge r \ge f(r)$ of Definition~\ref{def:F}, closing the loop with Theorem~\ref{thm:base_algo}.
\end{lemma}

\begin{proof}
See Appendix~\ref{app:proof_enclosure}.
\end{proof}

Finally, integrating Eq.~\ref{eq:ano_gain} in partial fractions yields an elementary antiderivative. With $E = e^{at}$, $t = x - x_0$:
\begin{equation}
\label{eq:ano_antiderivative}
    \Phi(t) = \frac{\kappa_{+}}{a}\left[ at - \frac{2\nu}{\nu-1}\ln(1+E) + \frac{\nu+1}{\nu-1}\ln(E+\nu) \right] \;\; (\nu \neq 1),
\end{equation}
and $\Phi(t) = \frac{\kappa_{+}}{a}\big[at - \ln(1+E) + \frac{2}{1+E}\big]$ at the removable singularity $\nu = 1$. Anchoring at $x=1$ to enforce $G(1)=1$:
\begin{equation}
\label{eq:ano_kernel}
    G(x) = 1 + \Phi\big(x - x_0\big) - \Phi\big(1 - x_0\big).
\end{equation}
Every ingredient of Eq.~\ref{eq:ano_kernel} is elementary: no special functions, no numerical quadrature, no iterative solve. The complete evaluation chain is
\begin{equation}
\label{eq:closed_form_chain}
    (\epsilon, \kappa_{+}, \kappa_{-}) \;\longrightarrow\; d \;\longrightarrow\; s \;\longrightarrow\; \nu \;\longrightarrow\; (B, C) \;\longrightarrow\; q \;\longrightarrow\; a \;\longrightarrow\; \big(x_0,\, E(\cdot),\, x_p\big),
\end{equation}
computed once before training; during training, evaluating $G$, $G'$, or $G''$ costs $O(1)$ per sample. Proofs of Proposition~\ref{prop:unique_inflection}, Lemma~\ref{lem:depth}, and Proposition~\ref{prop:scale} are given in Appendix~\ref{app:proof_ano}.

\subsection{Geometry and Induced Dynamics}
\label{sec:ano_analysis}

Figure~\ref{fig:kernel_geometry} visualizes the construction: $G$ is tangent to the identity at $(1,1)$ and peaks exactly at $1+\epsilon$, so maximizing the objective drives the ratio toward the peak: an \textbf{implicit trust region}. Panel (b) exposes the gain fields (PPO's dead zone, SPO's unbounded linear gain, ANO's bounded redescending pull) and panel (c) the resulting ratio-space flow (Eq.~\ref{eq:velocity}): ANO's velocity changes sign at $r = 1+\epsilon$ (a \emph{stable fixed point}), PPO's flow stops (neutral drift), and SPO's flow stiffens away from equilibrium.

\begin{figure*}[t]
    \centering
    \includegraphics[width=0.75\linewidth]{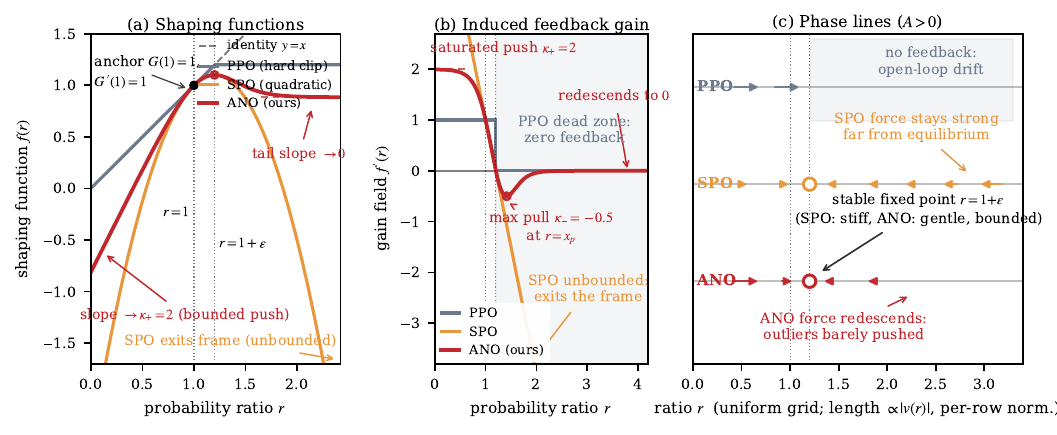}
    \caption{\textbf{Geometry and induced dynamics of surrogate objectives} (display config $\epsilon=0.2$, $\kappa_{+}=2$, $\kappa_{-}=-0.5$, chosen for visual clarity; the qualitative shape is identical for every feasible $(\epsilon, \kappa_{+}, \kappa_{-})$). (a) Shaping functions $f(r)$: PPO (hard clip), SPO (quadratic), and ANO (ours), all tangent to the identity $y=x$ at the anchor $r=1$ and peaking at $1+\epsilon$. ANO's left tail falls with bounded slope $\to \kappa_{+}$ (saturated push) while its right tail flattens with slope $\to 0$ toward the dashed asymptote $G(+\infty)$ (redescending pull); SPO is unbounded and exits the frame. (b) The induced feedback gain $\kappa(r) = f'(r)$: PPO implements a dead zone (zero feedback beyond $1+\epsilon$, shaded), SPO an unbounded linear gain that exits the frame, and ANO a bounded force in $[\kappa_{-}, \kappa_{+})$ that attains its pull floor $\kappa_{-}$ at $r=x_p$ and redescends to zero for extreme outliers. (c) Ratio-space phase lines for $A>0$ (arrow direction $=$ sign of $\dot r \propto A\,\kappa(r)$, length $\propto |\dot r|$, normalized per row; panel b gives the true scale): SPO and ANO both possess a stable fixed point at $r = 1+\epsilon$, but SPO's restoring force grows without bound while ANO's weakens past the fixed point and redescends, so extreme outliers are barely pushed; PPO's arrows stop exactly at the boundary and its flow stalls in the dead zone (open-loop drift under momentum).}
    \label{fig:kernel_geometry}
\end{figure*}

Two features deserve emphasis. \textbf{Bounded stiffness.} Since $G''$ is globally bounded, the per-sample velocity field has a bounded Lipschitz constant: the discrete-time update remains non-oscillatory under a wide range of step sizes, the property stress-tested in Section~\ref{sec:robustness}. \textbf{Redescending influence.} As $r \to +\infty$, the pull decays like $\kappa_{+}\nu/E$: extreme outliers, whose advantage estimates are least reliable, exert vanishing influence, in the spirit of redescending M-estimators \citep{huber2011robust}.

\paragraph{The objective.} ANO instantiates the symmetric objective of Definition~\ref{def:F} with $f = G$:
\begin{equation}
\label{eq:ano_objective}
    \mathcal{L}^{\mathrm{ANO}}(\tp) = \mathbb{E}_{\substack{s\sim\rho_{\pi} \\ a\sim\pi}} \Big[ \min\big( G(r)\,A_\pi(s,a),\; \big[2 - G(2-r)\big]\,A_\pi(s,a) \big) \Big].
\end{equation}
Implementation requires only replacing the clipped ratio in a PPO codebase with the element-wise evaluation of Eq.~\ref{eq:ano_kernel} (Appendix~\ref{app:sec:pseudocode}); all constants in the chain~\ref{eq:closed_form_chain} are precomputed once.

\section{Experiments}
\label{sec:experiments}

We empirically evaluate ANO by addressing three Research Questions: \textbf{RQ 1 (Performance \& consistency):} Does ANO deliver top-tier performance \emph{consistently} across domains and metrics? \textbf{RQ 2 (Structural robustness):} Does the bounded, redescending gain field tolerate aggressive hyperparameters, as predicted? \textbf{RQ 3 (Sensitivity):} How do the three knobs affect performance?

\subsection{Experimental Setup}
\label{sec:exp_setup}

\textbf{Benchmarks.} We evaluate on the full \textbf{Atari} benchmark ($40$ games)~\citep{bellemare2013arcade} and $6$ \textbf{MuJoCo} continuous-control tasks~\citep{todorov2012mujoco} (HalfCheetah, Ant, Hopper, Walker2d, Humanoid, Swimmer). Scores are normalized per environment (Human-Normalized Score for Atari; Expert-Normalized Score with a TD3 expert~\citep{pmlr-v80-fujimoto18a} for MuJoCo); a run's final score is the mean over its last $2\%$ of training steps, and scores are aggregated across $5$ seeds using \textit{rliable}~\citep{agarwal2021deep}: we report the Interquartile Mean (IQM), Mean, Median, and standard deviation of the normalized scores.

\textbf{Baselines and implementations.} We compare against \textbf{PPO}~\citep{schulman2017proximal}, \textbf{TRPO}~\citep{schulman2015trust}, \textbf{SPO}~\citep{xie2025simple}, and \textbf{PAPO}~\citep{zhao2024absolute}, all implemented in a single shared CleanRL-style codebase~\citep{huang2022cleanrl} with \texttt{EnvPool}~\citep{weng2022envpool} for parallel simulation, so all methods share identical networks, data pipelines, and evaluation protocols. Each baseline is run at several hyperparameter settings (clip/KL/$\omega$ grids in Appendix~\ref{app:sec:hyperparams}); for ANO we fix $(\kappa_{+}, \kappa_{-}) = (10, -1.5)$ and sweep $\epsilon \in \{0.1, 0.2, 0.3\}$. Full hyperparameters are deferred to Appendix~\ref{app:sec:details}.

\subsection{Main Results: Consistent Top-Rank Performance Across Domains}
\label{sec:main_results}

Tables~\ref{tab:atari_results} and~\ref{tab:mujoco_results} report the aggregate results on Atari and MuJoCo, respectively.

\begin{table}[t]
\centering
\caption{\textbf{Atari (40 games, 5 seeds).} Aggregate Human-Normalized Score (HNS) statistics, best configuration per method (full sweeps in Table~\ref{tab:atari_full}); ANO ranks first in IQM and Median.}
\label{tab:atari_results}
\small
\setlength{\tabcolsep}{3pt}
\begin{tabular}{lcccc}
\toprule
\textbf{Method} & \textbf{IQM} $\uparrow$ & \textbf{Mean HNS} $\uparrow$ & \textbf{Median HNS} $\uparrow$ & \textbf{Std HNS} \\
\midrule
\textbf{ANO} ($\epsilon{=}0.2$, $\kappa_{+}{=}10$, $\kappa_{-}{=}{-}1.5$) & \textbf{0.9350} & 3.8600 & \textbf{0.8084} & 9.2332 \\
PAPO ($\omega_1{=}\omega_2{=}0.005$) & 0.9290 & 4.0044 & 0.7239 & 9.8621 \\
SPO ($\epsilon{=}0.2$) & 0.9012 & 3.9109 & 0.7551 & 9.4342 \\
PPO ($\epsilon{=}0.1$) & 0.8873 & 3.8874 & 0.7199 & 10.0076 \\
TRPO (KL${=}0.01$) & 0.3462 & 2.5543 & 0.2405 & 9.3490 \\
ANO ($\epsilon{=}0.1$ / $0.3$, same $\kappa$) & 0.8698 / 0.8252 & 3.9211 / \textbf{5.5272} & 0.7665 / 0.6132 & 10.8456 / 18.1627 \\
\bottomrule
\end{tabular}
\end{table}

\begin{table}[t]
\centering
\caption{\textbf{MuJoCo (6 environments, 5 seeds).} Aggregate Expert-Normalized Score (ENS) statistics at the default learning rate $3\times10^{-4}$, best configuration per method (full sweeps in Table~\ref{tab:mujoco_full}); ANO ranks first in IQM, Mean, and Median. Bootstrap $95\%$ IQM CIs (per-environment stratified): ANO $[0.88, 0.99]$, SPO $[0.85, 0.98]$, PPO $[0.84, 0.98]$, TRPO $[0.78, 0.88]$, PAPO $[0.76, 0.83]$; ANO's interval has the highest lower bound.}
\label{tab:mujoco_results}
\small
\setlength{\tabcolsep}{3pt}
\begin{tabular}{lcccc}
\toprule
\textbf{Method} & \textbf{IQM} $\uparrow$ & \textbf{Mean ENS} $\uparrow$ & \textbf{Median ENS} $\uparrow$ & \textbf{Std ENS} \\
\midrule
\textbf{ANO} ($\epsilon{=}0.3$, $\kappa_{+}{=}10$, $\kappa_{-}{=}{-}1.5$) & \textbf{0.9304} & \textbf{0.9605} & \textbf{1.0079} & 0.5121 \\
SPO ($\epsilon{=}0.2$) & 0.9249 & 0.9139 & 0.9780 & 0.4849 \\
PPO ($\epsilon{=}0.2$) & 0.9185 & 0.8662 & 0.9222 & 0.4089 \\
TRPO (KL${=}0.01$) & 0.8501 & 0.8204 & 0.8249 & 0.4060 \\
PAPO ($\omega_1{=}\omega_2{=}0.005$) & 0.7925 & 0.8112 & 0.8723 & 0.4504 \\
ANO ($\epsilon{=}0.1$ / $0.2$, same $\kappa$) & 0.8287 / 0.8764 & 0.8261 / 0.9057 & 0.7377 / 0.8791 & 0.3626 / 0.5225 \\
\bottomrule
\end{tabular}
\end{table}

\textbf{Answering RQ 1.} ANO attains the \emph{best IQM and best Median on both domains} (the two robust aggregate statistics recommended by \citet{agarwal2021deep}) and additionally the best Mean on MuJoCo. The runner-up is \emph{domain-dependent} (PAPO on Atari, SPO on MuJoCo), so no single baseline matches ANO on both domains; ANO's only non-top aggregate is the Atari Mean, dominated by a few heavy-tailed games --- the very reason \citet{agarwal2021deep} recommend IQM and Median over the raw average. The ANO--PAPO Atari IQM gap is statistically indistinguishable from zero (paired bootstrap $95\%$ CI $[-0.10, 0.11]$); ANO ranks first or second at every evaluation window (Table~\ref{tab:window_sweep}), and averaged over the \emph{entire} run ANO's IQM exceeds PAPO's by $0.072$ ($P>0.999$). TRPO, the only explicitly KL-constrained method, ranks last on Atari and mid-pack on MuJoCo, further evidence that heavy constraint machinery is unnecessary.

\subsection{Structural Robustness: Learning-Rate Stress Test}
\label{sec:robustness}

The dynamical analysis of Section~\ref{sec:gainfield} predicts that the bounded gain field of ANO yields a larger stability region under aggressive step sizes, whereas PPO's dead zone (drift) and SPO's linearly stiffening gain (overshoot) should degrade. We test this prediction on MuJoCo by raising the learning rate from $3\times10^{-4}$ to $1\times10^{-3}$, keeping all else fixed.


\textbf{Answering RQ 2.} Table~\ref{tab:robustness} (Appendix~\ref{app:sec:full_experiments}) confirms the prediction quantitatively: ANO's degradation ($-0.9\%$) is $7\times$ smaller than SPO's ($-6.6\%$) and $60\times$ smaller than PPO's collapse ($-54.5\%$). The stressed ANO (IQM $0.9225$ at lr $10^{-3}$) still exceeds the best-tuned PPO ($0.9185$) and PAPO ($0.8160$) at their nominal learning rate. This is the empirical signature of a bounded-stiffness feedback law: the dynamics remains well-conditioned far outside the tuned hyperparameter regime.

\subsection{Dynamics Diagnostics: How Much Data Lives Outside the Trust Region?}
\label{sec:beyond_boundary}

Section~\ref{sec:gainfield} makes a directly testable prediction: samples whose ratio leaves the boundary receive zero gradient under PPO but bounded corrections under ANO. We monitor the resulting \emph{out-of-boundary fraction} (share of samples with $|r-1| > \epsilon$) at every update; Figures~\ref{fig:clipfrac_atari40} and~\ref{fig:clipfrac_asterix} (Appendix~\ref{app:sec:full_experiments}) show the Asterix trace and the cross-game distributions.

Three observations stand out from Figure~\ref{fig:clipfrac_asterix} and Table~\ref{tab:beyond_boundary} (Appendix~\ref{app:sec:full_experiments}). \textbf{(i) The dead zone is heavily populated under PPO.} At matched $\epsilon{=}0.1$, PPO leaves $7.88\%$ of samples outside the boundary with zero gradient on average; ANO keeps $1.21\%$ ($6.5\times$ smaller, statistically matching SPO's $1.22\%$); at a \emph{wider} boundary ($\epsilon{=}0.2$) ANO keeps $3.92\%$, the restoring force pulling ratios back while PPO's dead zone lets them diffuse. \textbf{(ii) The fraction moves with $\epsilon$ in the direction predicted by the gain-field picture.} ANO's share grows monotonically with $\epsilon$ at a fixed threshold ($1.21\% \to 3.92\% \to 5.62\%$), since a later zero-crossing delays the onset of the restoring force. \textbf{(iii) SPO attains an even smaller fraction, but through an \emph{unbounded} restoring gain}: the diagnostic measures where samples sit, not what force acts on them; force boundedness, not the fraction itself, determines robustness under aggressive hyperparameters (Section~\ref{sec:robustness}).


\subsection{Sensitivity to the Trust-Region Boundary}
\label{sec:sensitivity}

\textbf{Answering RQ 3.} With $(\kappa_{+}, \kappa_{-})$ fixed at $(10, -1.5)$, sweeping $\epsilon \in \{0.1, 0.2, 0.3\}$ keeps ANO competitive: Atari IQM $0.87$--$0.94$, MuJoCo $0.83$--$0.93$, with $\epsilon{=}0.2$ preferred on Atari and $\epsilon{=}0.3$ on MuJoCo. The knobs admit direct interpretation (Table~\ref{tab:knobs}); a fuller $(\kappa_{+}, \kappa_{-})$ study is future work.

\subsection{Extension to LLM Alignment}
\label{sec:llm}

Finally, we test whether the gain-field advantage transfers to RLHF: fine-tuning a language model against a learned reward model. We use Reddit TL;DR summarization~\citep{stiennon2020learning} with a Pythia-1B policy and a Pythia-1B reward model ($\beta = 0.05$), trained for $1$M episodes under the \texttt{TRL} pipeline~\citep{vonwerra2020trl}; ANO replaces the clipped objective by Eq.~\ref{eq:ano_objective} without any other modification (one run per method; we therefore treat these results as \emph{suggestive}; full hyperparameters in Appendix~\ref{app:sec:details}).


\textbf{KL dynamics.} The same signature appears in classical RL: PPO's KL from the SFT reference never saturates, whereas ANO's restoring force contains the drift at the lowest final KL ($39.5$ vs.\ $54.0$ for PPO; SPO: $40.9$); ANO also preserves the most generation diversity (final entropy $23.7$ vs.\ $21.4$); bounded updates do not collapse the policy (Appendix Figure~\ref{fig:rlhf_dynamics}).

\textbf{Quality evaluation without the training reward model.} The training reward model is a compromised evaluator for exactly the methods it trained, so we score generations with an independent LLM judge instead: each summary is rated on a $1$--$10$ rubric (relevance, coherence, faithfulness, conciseness) in isolation to eliminate position bias, on $300$ held-out test prompts, reporting paired per-prompt differences with bootstrap CIs and a TOST equivalence test ($\delta{=}0.5$); greedy decoding is primary, with sampling at $T\in\{0.7,1.0\}$ as a robustness check.

\textbf{Results.} Figure~\ref{fig:rubric_forest} and Table~\ref{tab:rlhf_results} (both Appendix~\ref{app:sec:full_experiments}) give three findings. \textbf{(i) ANO beats PPO on quality and stability simultaneously:} at $T{=}0$ the paired overall difference is $+0.83$ [$+0.63, +1.02$] on the $10$-point scale (sign test, $182$ vs.\ $78$ of $260$ non-tied pairs, $p < 10^{-9}$), driven by faithfulness ($+1.36$) and relevance ($+1.15$); equal output lengths ($192$ vs.\ $188$ characters) rule out a verbosity artifact. \textbf{(ii) ANO matches the strongest baseline:} against SPO it is $+0.14$ [$-0.01, +0.29$], TOST-equivalent at all three temperatures, achieved with a \emph{bounded} gain where SPO's unbounded force degrades (Section~\ref{sec:robustness}). \textbf{(iii) The advantage is robust to the push/pull knobs:} halving the gain magnitudes ($(\kappa_{+}, \kappa_{-}) = (2, -0.5)$) preserves both conclusions (vs.\ PPO: $+0.72$ [$+0.54, +0.90$]; vs.\ SPO: equivalent), the final KL shifting toward SPO's level as predicted (Table~\ref{tab:knobs}).


\section{Conclusion}
\label{sec:conclusion}

We presented a dynamical view: surrogate objectives are feedback laws on the probability ratio whose stability is governed by the induced gain field, separating necessary ingredients (identity anchoring, two-sided enclosure, bounded smooth gains) from accidental machinery (hard clipping, second-order constraints). This yielded Anchored Neighborhood Optimization (ANO): a derivative-domain shaping kernel with three decoupled knobs and a closed-form parameterization. ANO ranks first on both domains (IQM, Median), degrading only $0.9\%$ at $3.3\times$ the learning rate; future work includes dynamic trust regions, asymmetric extensions, and GRPO-style integration~\citep{shao2024deepseekmath}.

\subsection*{AI use statement}
In this work, we used generative AI tools in the following ways. (1) \textbf{Writing:}
an AI assistant (Kimi) helped polish the language of the manuscript. (2)
\textbf{Retrieval and discovery:} AI-assisted search was used to locate related
work and reference data, including the provenance of the normalization constants
reported in Appendix~\ref{app:sec:details}. (3) \textbf{Code:} AI assistants
helped develop, debug, and refactor the experimental codebase and the analysis
scripts. In addition, an LLM (DeepSeek-V4.1-Flash, API model id \texttt{deepseek-chat}, accessed via its public API)
served as the independent judge for the RLHF evaluation in
Section~\ref{sec:llm}. We have not used generative AI tools for running
experiments, collecting results, or fitting data. We have reviewed all
AI-assisted work and take responsibility for the final content of this work,
including text, claims or artifacts produced with the aid of generative AI.

\subsection*{Reproducibility statement}
A complete proof of the theoretical claims is provided in the appendix.
Experimental details, including hyperparameters and evaluation protocols,
are described in Section~\ref{sec:experiments} and Appendix~\ref{app:sec:details}.
All experiments are conducted with publicly available benchmarks and libraries.

\bibliography{example_paper}
\bibliographystyle{iclr2027_conference}


\appendix

\newpage
\section{Experimental Details}
\label{app:sec:details}

In this section, we provide the comprehensive implementation details, hyperparameter settings, and computational resources used in our experiments to ensure reproducibility.

\subsection{Implementation Frameworks}
Our experiments are built upon established open-source libraries to ensure reproducibility and high efficiency:
\begin{itemize}[leftmargin=*, itemsep=0.5mm, topsep=0pt]
    \item \textbf{Traditional RL (MuJoCo \& Atari):} All algorithms (PPO, TRPO, SPO, PAPO, ANO) are implemented in a \emph{single shared} CleanRL-style codebase~\citep{huang2022cleanrl}, so that every method uses identical network architectures, data pipelines, advantage estimators, and evaluation protocols. We utilize \texttt{EnvPool}~\cite{weng2022envpool} for massive parallel environment simulation, which significantly accelerates the wall-clock training time compared to standard \texttt{Gymnasium}~\cite{towers2024gymnasium} wrappers.
    \item \textbf{LLM Fine-tuning:} We use the \texttt{TRL} library~\citep{vonwerra2020trl} integrated with \texttt{Accelerate}~\cite{accelerate} and \texttt{DeepSpeed}~\cite{rasley2020deepspeed}. The reward model is a fine-tuned Pythia-1B scoring model, consistent with the standard CleanRL setup for efficient 1B-scale experiments. Table~\ref{tab:llm_hyperparams} lists the hyperparameters; generation quality is evaluated by an independent LLM judge (DeepSeek-V4.1-Flash, JSON mode, temperature $0$) scoring each summary on a $1$--$10$ rubric (relevance, coherence, faithfulness, conciseness) over $300$ held-out TL;DR test prompts, at decoding temperatures $\{0, 0.7, 1.0\}$ with $53$ new tokens.
\end{itemize}

\subsection{Normalization and Aggregation}
For Atari, we apply Human Normalized Score (HNS):
\begin{equation}
    \text{HNS}=\frac{\text{Agent Score} - \text{Random Score}}{\text{Human Score} - \text{Random Score}},
\end{equation}
where the per-game random and human reference scores follow \citet{mnih2015human}; ElevatorAction, Carnival, and AirRaid use the reference scores of \citet{vieillard2020munchausen} (their Table~3): $(0, 3000)$, $(380, 4000)$, and $(400, 3000)$ in $(\mathrm{random}, \mathrm{human})$ form, respectively.

For MuJoCo, we apply Expert Normalized Score (ENS):
\begin{equation}
    \text{ENS}=\frac{\text{Agent Score} - \text{Random Score}}{\text{Expert Score} - \text{Random Score}},
\end{equation}
where the expert we used is TD3~\citep{pmlr-v80-fujimoto18a} just like~\citet{2025Hyperspherical} and the TD3 scores are from~\citet{tianshou} and random scores are from~\citet{2025Hyperspherical} (the Swimmer random score is set to $0$, as its random performance is negligible).

Aggregate statistics (IQM, Mean, Median, Std) are computed over the per-environment normalized scores of $5$ seeds using the \textit{rliable} library~\citep{agarwal2021deep}; a run's final score is the mean over its last $2\%$ of training steps.

\begin{table}[t]
\centering
\caption{\textbf{Evaluation-window sweep} (best configuration per method): a run's final score is the mean over the last $x\%$ of training. The Atari ANO--PAPO IQM gap stays within $\pm 0.02$ through the $50\%$ window and the two methods are statistically indistinguishable (paired bootstrap, $95\%$ CI $[-0.10, 0.11]$ at the $2\%$ window); averaged over the entire run ($100\%$), ANO's Atari IQM exceeds PAPO's by $0.072$, where $P$ is the paired-bootstrap probability that ANO's IQM exceeds PAPO's ($P>0.999$).}
\label{tab:window_sweep}
\scriptsize
\setlength{\tabcolsep}{2.5pt}
\renewcommand{\arraystretch}{0.80}
\begin{tabular}{lccccc|ccccc}
\toprule
& \multicolumn{5}{c|}{\textbf{Atari (HNS)}} & \multicolumn{5}{c}{\textbf{MuJoCo (ENS)}} \\
\textbf{Window} & \textbf{ANO} & \textbf{PAPO} & \textbf{SPO} & \textbf{PPO} & \textbf{TRPO} & \textbf{ANO} & \textbf{SPO} & \textbf{PPO} & \textbf{TRPO} & \textbf{PAPO} \\
\midrule
\multicolumn{6}{c|}{\emph{IQM}} & \multicolumn{5}{c}{\emph{IQM}} \\
\midrule
$2\%$  & \textbf{0.9350} & 0.9290 & 0.9012 & 0.8873 & 0.3462 & \textbf{0.9304} & 0.9249 & 0.9185 & 0.8501 & 0.7925 \\
$10\%$ & 0.9137 & \textbf{0.9188} & 0.8861 & 0.9003 & 0.3527 & \textbf{0.9413} & 0.9312 & 0.9082 & 0.8476 & 0.7805 \\
$50\%$ & \textbf{0.8100} & 0.7909 & 0.7827 & 0.7721 & 0.3382 & 0.8991 & \textbf{0.9090} & 0.8716 & 0.7355 & 0.7477 \\
$100\%$ & \textbf{0.5610} & 0.4889 & 0.5288 & 0.4879 & 0.2931 & 0.5334 & 0.5173 & \textbf{0.5645} & 0.4364 & 0.4124 \\
\midrule
\multicolumn{6}{c|}{\emph{Mean}} & \multicolumn{5}{c}{\emph{Mean}} \\
\midrule
$2\%$  & 3.8600 & \textbf{4.0044} & 3.9109 & 3.8874 & 2.5543 & \textbf{0.9605} & 0.9139 & 0.8662 & 0.8204 & 0.8112 \\
$10\%$ & 3.8613 & \textbf{3.8698} & 3.7543 & 3.6812 & 2.5669 & \textbf{0.9652} & 0.9142 & 0.8571 & 0.8177 & 0.7936 \\
$50\%$ & \textbf{3.4650} & 3.1869 & 3.2575 & 3.0810 & 2.2840 & \textbf{0.8999} & 0.8579 & 0.8270 & 0.7468 & 0.7370 \\
$100\%$ & \textbf{1.7413} & 1.6515 & 1.7190 & 1.6424 & 1.1040 & 0.5276 & 0.5014 & \textbf{0.5686} & 0.4737 & 0.4339 \\
\midrule
\multicolumn{6}{c|}{\emph{Median}} & \multicolumn{5}{c}{\emph{Median}} \\
\midrule
$2\%$  & \textbf{0.8084} & 0.7239 & 0.7551 & 0.7199 & 0.2405 & \textbf{1.0079} & 0.9780 & 0.9222 & 0.8249 & 0.8723 \\
$10\%$ & \textbf{0.8093} & 0.6854 & 0.7341 & 0.6926 & 0.2298 & \textbf{1.0098} & 1.0024 & 0.8845 & 0.8200 & 0.8420 \\
$50\%$ & \textbf{0.6745} & 0.5899 & 0.6539 & 0.5692 & 0.2355 & 0.9716 & \textbf{0.9746} & 0.8272 & 0.7507 & 0.8086 \\
$100\%$ & \textbf{0.4776} & 0.4184 & 0.4514 & 0.3904 & 0.2199 & 0.5545 & 0.5152 & \textbf{0.5815} & 0.4559 & 0.3600 \\
\bottomrule
\end{tabular}
\renewcommand{\arraystretch}{1.0}
\end{table}

\subsection{Evaluation Protocol and Environment Preprocessing}
\textbf{Atari.} We use the ALE \texttt{v5} environments via \texttt{EnvPool} with the standard evaluation preprocessing: $84\times84$ grayscale frames, frame skip $4$, frame stack $4$, noop-max $30$, and episodic life. Reward \emph{clipping} (sign) is applied only to the training rewards; raw (unclipped) episode returns are logged for evaluation. Episodic returns are recorded only upon true game over (all lives exhausted). For \texttt{Atlantis-v5}, whose default episode length is unbounded under good play, we additionally cap the episode at $27{,}000$ steps to guarantee termination (a different purpose from the $30$k evaluation cap discussed in Appendix~\ref{app:sec:full_experiments}, which makes non-terminating TRPO episodes measurable rather than bounding unbounded ones).
\textbf{MuJoCo.} We use the \texttt{v4} environments (HalfCheetah, Ant, Hopper, Walker2d, Humanoid, Swimmer) with default termination conditions.

\subsection{Hyperparameter Settings}
\label{app:sec:hyperparams}

All methods share the common training configuration within each domain; algorithm-specific hyperparameters are listed below.

\begin{table}[h]
    \centering
    \caption{\textbf{Shared training configuration for traditional RL experiments.}}
    \label{tab:shared_hyperparams}
    \setlength{\tabcolsep}{4.5pt}
    \begin{tabular}{@{}lcc@{}}
        \toprule
        \textbf{Hyperparameter} & \textbf{Atari} & \textbf{MuJoCo} \\
        \midrule
        Total timesteps        & $6\times10^{6}$ & $1\times10^{7}$ \\
        Learning rate (annealed) & $2.5\times10^{-4}$ & $3\times10^{-4}$ \\
        Num. parallel envs     & 16 & 128 \\
        Rollout steps per env  & 128 & 32 \\
        Update epochs          & 4 & 10 \\
        Num. minibatches       & 4 & 4 \\
        Discount $\gamma$      & 0.99 & 0.99 \\
        GAE $\lambda$          & 0.95 & 0.95 \\
        Entropy coef.          & 0.01 & 0.0 \\
        Value coef.            & 0.5 & 0.5 \\
        Max grad norm          & 0.5 & 0.5 \\
        Optimizer              & Adam ($\varepsilon=10^{-5}$) & Adam ($\varepsilon=10^{-5}$) \\
        Seeds                  & $\{1,2,3,4,5\}$ & $\{1,2,3,4,5\}$ \\
        \bottomrule
    \end{tabular}
\end{table}

\textbf{Algorithm-specific hyperparameters.} PPO and SPO use clip range $\epsilon \in \{0.1, 0.2, 0.3\}$ (grid search). TRPO uses maximal KL $\delta \in \{0.01, 0.02, 0.03\}$ (grid search), damping $0.1$, $10$ conjugate-gradient iterations, and $10$ backtracking line-search iterations with ratio $0.5$; its actor and critic use separate feature extractors so that value updates cannot violate the accepted KL step. PAPO uses $k = 7$, target KL $0.02$, the detailed bias-correction variant, and $(\omega_1, \omega_2) \in \{0.001, 0.005, 0.01\}^2$ (diagonal grid search). ANO fixes $(\kappa_{+}, \kappa_{-}) = (10, -1.5)$ (the best of a two-point grid $\{(2, -0.5), (10, -1.5)\}$ over push/pull magnitudes, i.e., halving or doubling the gain scale; a denser $(\kappa_{+}, \kappa_{-})$ grid is prohibitive, as each candidate pair alone costs $3$ boundary values $\times$ $5$ seeds $\times$ $40$ games $= 600$ Atari runs, and a single RLHF training run takes ${\sim}60$ hours) and sweeps the trust-region boundary $\epsilon \in \{0.1, 0.2, 0.3\}$. For the learning-rate stress test (Table~\ref{tab:robustness}), the learning rate is raised to $1\times10^{-3}$ with every other hyperparameter unchanged.

\begin{table}[h]
    \centering
    \caption{\textbf{Hyperparameters for LLM fine-tuning (TL;DR).} Global batch size and total episodes are aligned across methods; settings follow the standard \texttt{TRL} TL;DR setup~\citep{vonwerra2020trl}. ANO uses $\epsilon = 0.2$; the two $(\kappa_{+}, \kappa_{-})$ configurations correspond to the rows of Table~\ref{tab:rlhf_results}.}
    \label{tab:llm_hyperparams}
    \begin{tabular}{@{}lcc@{}}
        \toprule
        \textbf{Hyperparameter} & \textbf{PPO/SPO} & \textbf{ANO (Ours)} \\
        \midrule
        Base model / reward model & \multicolumn{2}{c}{Pythia-1B / Pythia-1B (fine-tuned)} \\
        Dataset & \multicolumn{2}{c}{\texttt{trl-lib/tldr}} \\
        Optimizer & \multicolumn{2}{c}{AdamW} \\
        Learning rate & \multicolumn{2}{c}{$3\times10^{-6}$ (cosine decay)} \\
        KL coefficient $\beta$ & \multicolumn{2}{c}{$0.05$} \\
        Per-device batch size / grad.\ accum. & \multicolumn{2}{c}{$8$ / $8$ (global $64$)} \\
        Update epochs & \multicolumn{2}{c}{$4$} \\
        Total steps / episodes & \multicolumn{2}{c}{$15{,}625$ / $1{,}000{,}000$} \\
        Clip range / boundary $\epsilon$ & $0.2$ & $0.2$ \\
        $(\kappa_{+}, \kappa_{-})$ &  & $(10, -1.5)$, $(2, -0.5)$ \\
        \bottomrule
    \end{tabular}
\end{table}


\subsection{Computational Resources}
Experiments were conducted on NVIDIA H100 GPUs. ANO's element-wise kernel evaluation adds negligible overhead relative to the neural network forward-backward passes; all first-order methods (PPO, SPO, ANO) share the same $\mathcal{O}(N)$ time complexity per update.

\clearpage
\section{Full Experimental Results}
\label{app:sec:full_experiments}

Tables~\ref{tab:atari_full} and~\ref{tab:mujoco_full} give the complete hyperparameter sweeps behind the main-text summaries (Tables~\ref{tab:atari_results} and~\ref{tab:mujoco_results}).

\begin{table}[H]
\centering
\caption{\textbf{Full Atari sweep (40 games, 5 seeds).} Aggregate Human-Normalized Score statistics for every hyperparameter configuration.}
\label{tab:atari_full}
\scriptsize
\setlength{\tabcolsep}{4pt}
\begin{tabular}{lcccc}
\toprule
\textbf{Method} & \textbf{IQM} $\uparrow$ & \textbf{Mean HNS} $\uparrow$ & \textbf{Median HNS} $\uparrow$ & \textbf{Std HNS} \\
\midrule
ANO ($\epsilon{=}0.2$, $\kappa_{+}{=}10$, $\kappa_{-}{=}{-}1.5$) & \textbf{0.9350} & 3.8600 & \textbf{0.8084} & 9.2332 \\
PAPO ($\omega_1{=}\omega_2{=}0.005$) & 0.9290 & 4.0044 & 0.7239 & 9.8621 \\
SPO ($\epsilon{=}0.2$) & 0.9012 & 3.9109 & 0.7551 & 9.4342 \\
PPO ($\epsilon{=}0.1$) & 0.8873 & 3.8874 & 0.7199 & 10.0076 \\
PAPO ($\omega_1{=}\omega_2{=}0.01$) & 0.8791 & 4.0994 & 0.6303 & 10.7148 \\
PAPO ($\omega_1{=}\omega_2{=}0.001$) & 0.9046 & 4.4863 & 0.6501 & 12.8591 \\
SPO ($\epsilon{=}0.1$) & 0.8748 & 3.6115 & 0.7358 & 9.0909 \\
ANO ($\epsilon{=}0.1$, $\kappa_{+}{=}10$, $\kappa_{-}{=}{-}1.5$) & 0.8698 & 3.9211 & 0.7665 & 10.8456 \\
ANO ($\epsilon{=}0.3$, $\kappa_{+}{=}10$, $\kappa_{-}{=}{-}1.5$) & 0.8252 & \textbf{5.5272} & 0.6132 & 18.1627 \\
PPO ($\epsilon{=}0.2$) & 0.8164 & 3.9020 & 0.6626 & 11.6492 \\
SPO ($\epsilon{=}0.3$) & 0.7674 & 4.9649 & 0.6198 & 17.0069 \\
PPO ($\epsilon{=}0.3$) & 0.7569 & 3.6087 & 0.6329 & 10.2429 \\
TRPO (KL${=}0.01$) & 0.3462 & 2.5543 & 0.2405 & 9.3490 \\
TRPO (KL${=}0.02$) & 0.2609 & 2.0484 & 0.1900 & 7.9825 \\
TRPO (KL${=}0.03$) & 0.1969 & 1.9433 & 0.1378 & 8.1475 \\
\bottomrule
\end{tabular}
\end{table}

\begin{table}[H]
\centering
\caption{\textbf{Full MuJoCo sweep (6 environments, 5 seeds).} Aggregate Expert-Normalized Score statistics at the default learning rate $3\times10^{-4}$ for every hyperparameter configuration.}
\label{tab:mujoco_full}
\scriptsize
\setlength{\tabcolsep}{4pt}
\begin{tabular}{lcccc}
\toprule
\textbf{Method} & \textbf{IQM} $\uparrow$ & \textbf{Mean ENS} $\uparrow$ & \textbf{Median ENS} $\uparrow$ & \textbf{Std ENS} \\
\midrule
ANO ($\epsilon{=}0.3$, $\kappa_{+}{=}10$, $\kappa_{-}{=}{-}1.5$) & \textbf{0.9304} & \textbf{0.9605} & 1.0079 & 0.5121 \\
SPO ($\epsilon{=}0.2$) & 0.9249 & 0.9139 & 0.9780 & 0.4849 \\
PPO ($\epsilon{=}0.2$) & 0.9185 & 0.8662 & 0.9222 & 0.4089 \\
PPO ($\epsilon{=}0.1$) & 0.9132 & 0.8976 & 0.9197 & 0.5002 \\
ANO ($\epsilon{=}0.2$, $\kappa_{+}{=}10$, $\kappa_{-}{=}{-}1.5$) & 0.8764 & 0.9057 & 0.8791 & 0.5225 \\
PPO ($\epsilon{=}0.3$) & 0.8683 & 0.8308 & 0.8320 & 0.3695 \\
SPO ($\epsilon{=}0.3$) & 0.8588 & 0.8900 & 0.9495 & 0.5603 \\
SPO ($\epsilon{=}0.1$) & 0.8610 & 0.8529 & 0.8167 & 0.3838 \\
TRPO (KL${=}0.01$) & 0.8501 & 0.8204 & 0.8249 & 0.4060 \\
ANO ($\epsilon{=}0.1$, $\kappa_{+}{=}10$, $\kappa_{-}{=}{-}1.5$) & 0.8287 & 0.8261 & 0.7377 & 0.3626 \\
PAPO ($\omega_1{=}\omega_2{=}0.001$) & 0.8160 & 0.8166 & \textbf{1.0112} & 0.4595 \\
TRPO (KL${=}0.02$) & 0.8053 & 0.7531 & 0.8319 & 0.3408 \\
PAPO ($\omega_1{=}\omega_2{=}0.005$) & 0.7925 & 0.8112 & 0.8723 & 0.4504 \\
PAPO ($\omega_1{=}\omega_2{=}0.01$) & 0.7508 & 0.7766 & 0.6144 & 0.4318 \\
\bottomrule
\end{tabular}
\end{table}

Per-environment training curves (mean $\pm$ std over $5$ seeds, binned every $100$k steps) are shown in Figures~\ref{fig:atari_curves}--\ref{fig:mujoco_curves}: all $40$ Atari games and all $6$ MuJoCo tasks, with the best-tuned configuration per method (ANO $(\epsilon{=}0.2, \kappa_{+}{=}10, \kappa_{-}{=}{-}1.5)$ on Atari and $(\epsilon{=}0.3, \kappa_{+}{=}10, \kappa_{-}{=}{-}1.5)$ on MuJoCo; PPO $\epsilon{=}0.1$; SPO $\epsilon{=}0.2$; PAPO $(0.005, 0.005)$; TRPO $\mathrm{KL}{=}0.01$). The top-group comparison is statistically robust: the ANO--PAPO Atari IQM difference is indistinguishable from zero (paired bootstrap $95\%$ CI $[-0.10, 0.11]$; ANO's lower bound is the highest) and the ranking is unchanged across evaluation windows (Table~\ref{tab:window_sweep}).

Two reporting caveats for the Atari curves. (i) In ElevatorAction, episode terminations become rare once a policy learns to survive, so the episodic-return curves of \emph{all} methods are sparse and effectively shorter than the $6$M-step horizon for most seeds. (ii) All TRPO runs were re-trained to the full $6$M-step horizon after we found that TRPO policies on ElevatorAction and Tutankham can stop terminating episodes altogether (their episodic return then becomes undefined rather than low); a uniform episode cap of $30$k steps makes such runs measurable. The re-run TRPO curves are complete, and the cap only ever truncates TRPO's own ElevatorAction episodes (often at a near-zero return, as the surviving policy stops scoring. These issues only deflate the TRPO baseline; Table~\ref{tab:atari_results} evaluates the final policy of every method under an identical protocol, and all of our conclusions against TRPO are therefore conservative.

\section{Algorithm Implementation}
\label{app:sec:pseudocode}

\begin{algorithm}[H]
	\caption{Training Procedure for Anchored Neighborhood Optimization (ANO)}\label{algo:ano_training}
	\footnotesize
	\begin{algorithmic}[1]		\STATE {\bfseries Require:} Initial parameters $\theta_0, \phi_0$; ANO knobs $(\epsilon, \kappa_{+}, \kappa_{-})$; coefficients $\lambda_{val}, \lambda_{ent}$
		\STATE {\bfseries Hyperparameters:} Learning rates $\alpha_\theta, \alpha_\phi$; Discount $\gamma$, GAE parameter $\lambda$
		\STATE {\color{teal} \textit{\# 0. One-time closed-form setup}}
		\STATE Solve $(\nu, a, x_0)$ from $(\epsilon, \kappa_{+}, \kappa_{-})$ via the chain in Eq.~\eqref{eq:closed_form_chain} \COMMENT{no iterative solve}
		\FOR{iteration $k = 0, 1, 2, \dots$}
		    \STATE {\color{teal} \textit{\# 1. Interaction \& Data Collection}}
    		\STATE Sample trajectories $\mathcal{T} = \{(s_t, a_t, r_t)\}$ by running policy $\pi_{\theta_k}$ in the environment

    		\STATE {\color{teal} \textit{\# 2. Advantage Estimation}}
    		\STATE Compute generalized advantages $\hat{A}_t$ and value targets $\hat{R}_t$ using GAE\,($\gamma, \lambda$) based on $V_{\phi_k}$

    		\STATE {\color{teal} \textit{\# 3. Policy Snapshot}}
    		\STATE $\theta_{\text{old}} \leftarrow \theta_k$

    		\STATE {\color{teal} \textit{\# 4. Optimization Epochs}}
    		\FOR{epoch $m = 1, \dots, M$}
    		    \STATE Resample mini-batches from $\mathcal{T}$
    		    \FOR{each mini-batch $B$}
    		        \STATE Calculate probability ratio $r_t(\theta) = \frac{\pi_\theta(a_t|s_t)}{\pi_{\theta_{\text{old}}}(a_t|s_t)}$

    		        \STATE {\color{teal} \textit{\# Compute Unified Ratio Objective Loss}}
    		        \STATE Evaluate shaping function $G(r_t)$ via Eq.~\eqref{eq:ano_kernel} \COMMENT{$O(1)$ per sample}
    		        \STATE Derive symmetric dual $g(r_t) \leftarrow 2 - G(2-r_t)$
    		        \STATE $\mathcal{L}_{\text{policy}} \leftarrow - \frac{1}{|B|} \sum_{t \in B} \min \left( g(r_t)\hat{A}_t, \; G(r_t)\hat{A}_t \right)$

    		        \STATE {\color{teal} \textit{\# Auxiliary Losses}}
    		        \STATE $\mathcal{L}_{\text{ent}} \leftarrow \frac{1}{|B|} \sum_{t \in B} \mathcal{H}(\pi_\theta(\cdot|s_t))$
    		        \STATE $\mathcal{L}_{\text{val}} \leftarrow \frac{1}{2|B|} \sum_{t \in B} (V_\phi(s_t) - \hat{R}_t)^2$

    		        \STATE {\color{teal} \textit{\# Gradient Step}}
    		        \STATE $\mathcal{L}_{\text{total}} \leftarrow \mathcal{L}_{\text{policy}} + \lambda_{\text{val}} \mathcal{L}_{\text{val}} - \lambda_{\text{ent}} \mathcal{L}_{\text{ent}}$
    		        \STATE Update $\theta, \phi$ w.r.t. $\nabla \mathcal{L}_{\text{total}}$ via optimizer
    		    \ENDFOR
    		\ENDFOR
    		\STATE $\theta_{k+1} \leftarrow \theta, \enspace \phi_{k+1} \leftarrow \phi$
		\ENDFOR
	\end{algorithmic}
	\normalsize
\end{algorithm}

\begin{figure}[p]
    \centering
    \includegraphics[width=0.8\linewidth]{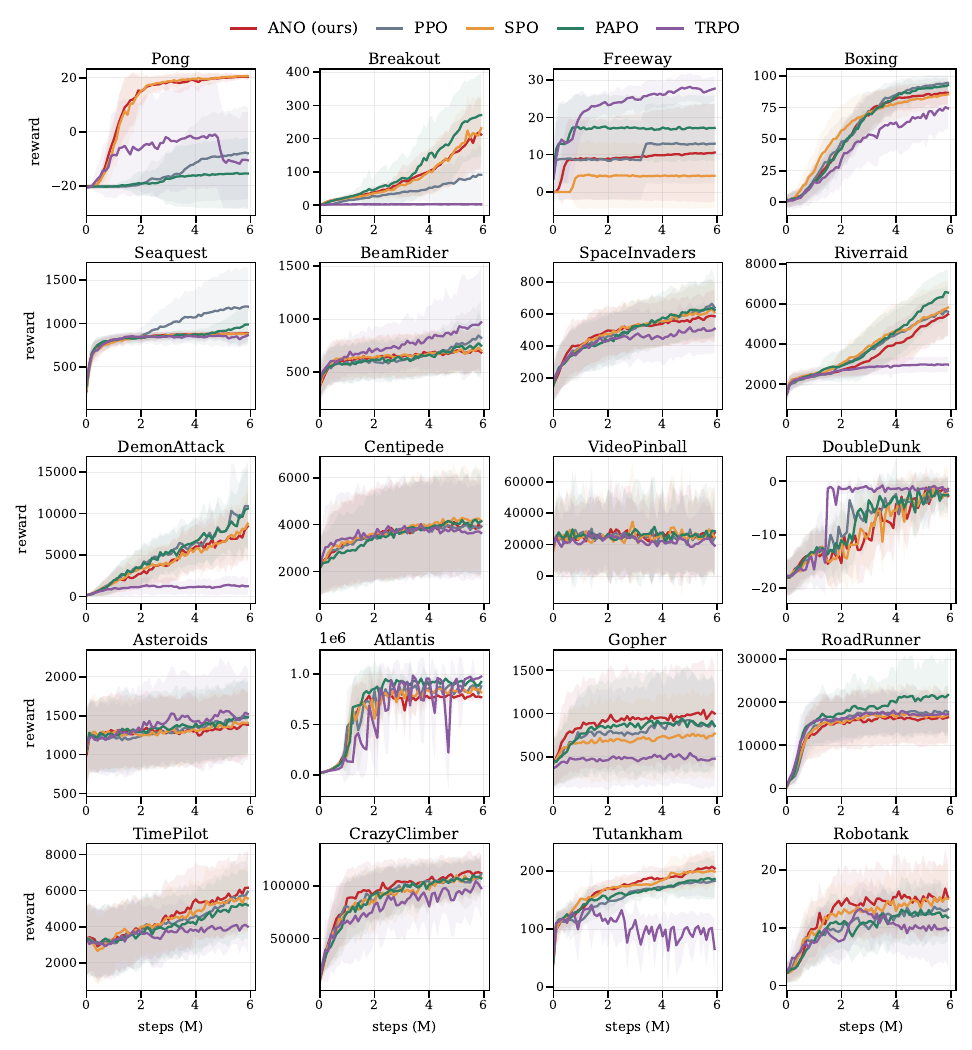}

    \smallskip
    \includegraphics[width=0.8\linewidth]{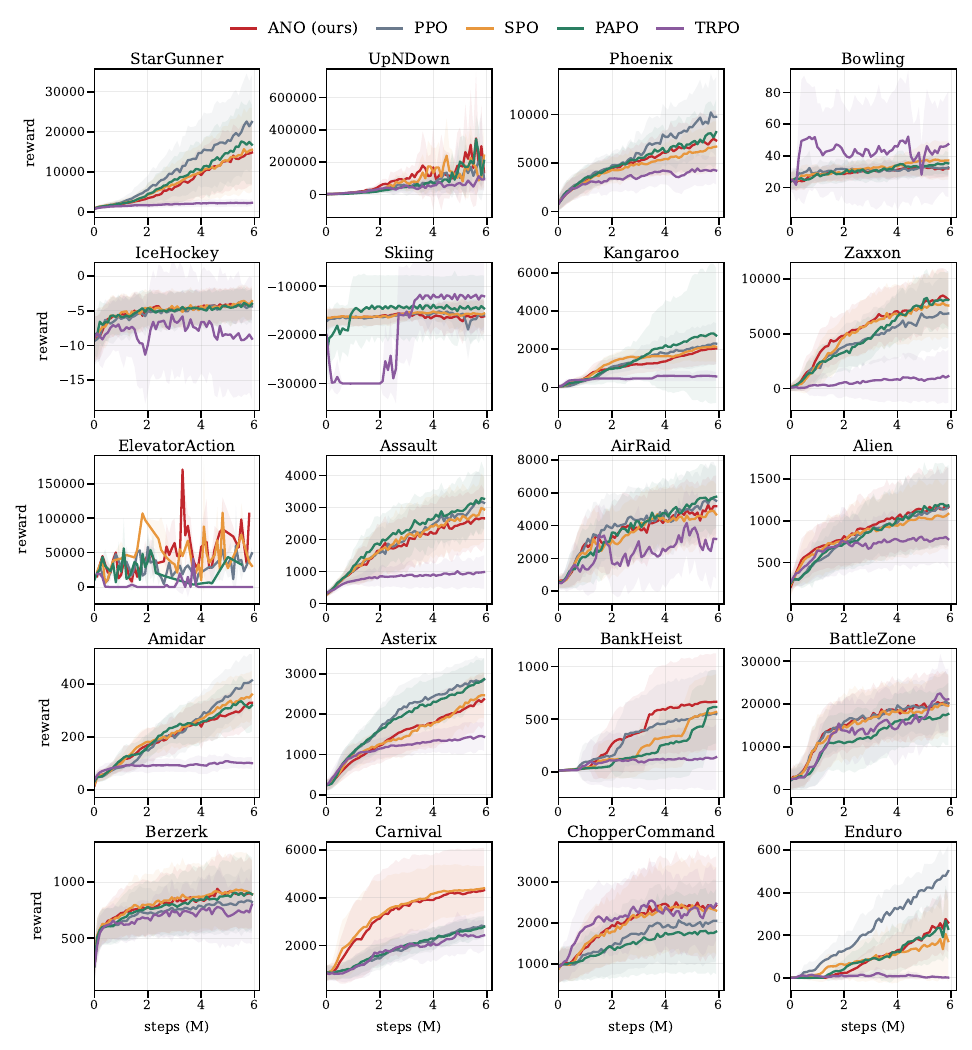}
    \caption{\textbf{Training curves on Atari (40 games).} Episode return (mean $\pm$ std, $5$ seeds) vs.\ environment steps for the best-tuned configuration per method (part 1 above, part 2 below).}
    \label{fig:atari_curves}
\end{figure}

\begin{figure}[p]
    \vspace*{\fill}
    \centering
    \includegraphics[width=0.85\linewidth]{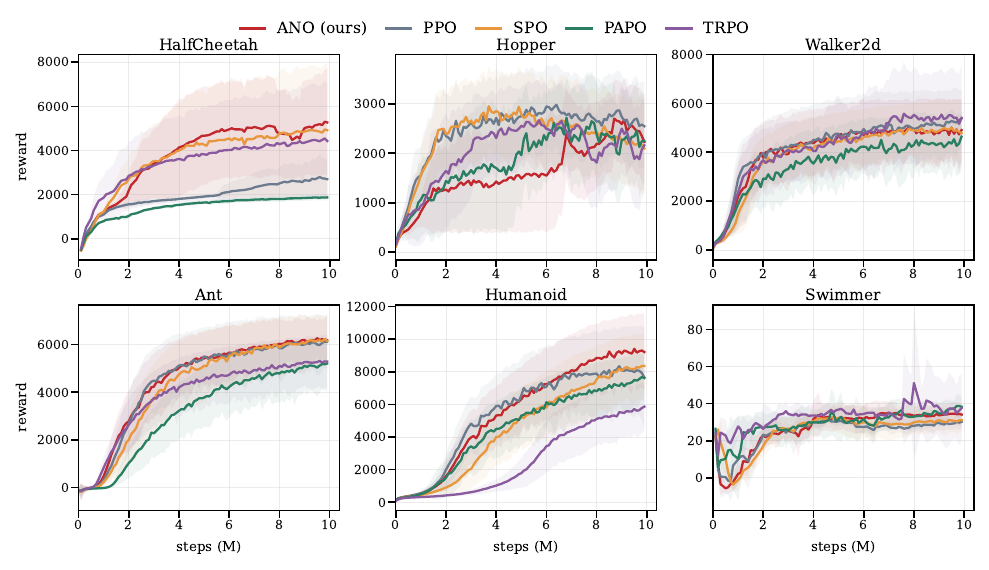}
    \caption{\textbf{Training curves on MuJoCo (6 environments).} Episode return (mean $\pm$ std, $5$ seeds) vs.\ environment steps for the best-tuned configuration per method.}
    \label{fig:mujoco_curves}
    \vspace*{\fill}
\end{figure}

\begin{figure}[t]
    \centering
    \includegraphics[width=0.70\linewidth]{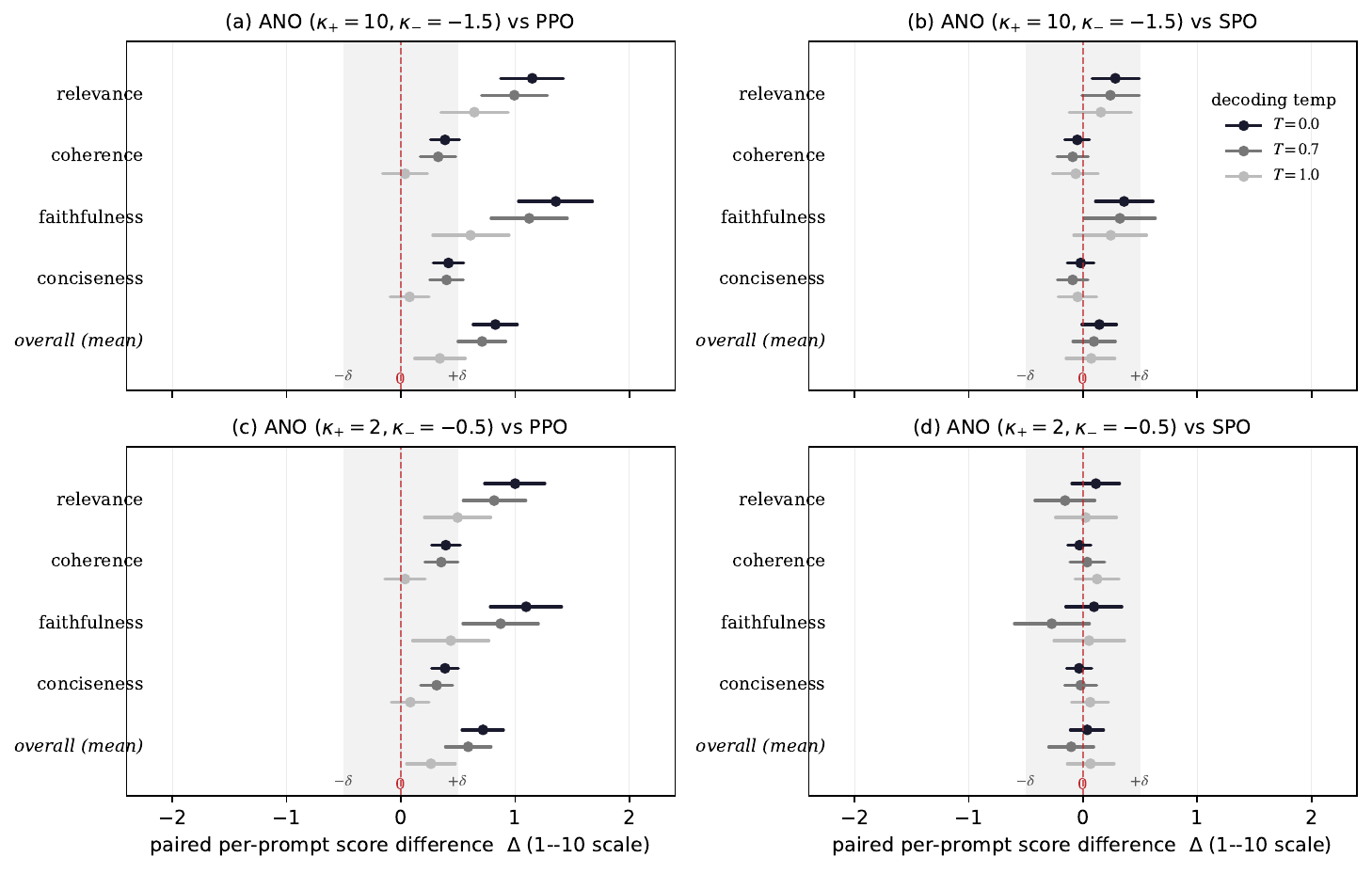}
    \caption{\textbf{Paired per-prompt rubric differences (ANO $-$ baseline), $N{=}300$.} Grey band: equivalence region $(-\delta, +\delta)$ of the TOST test; darker points mark lower decoding temperatures. (a,b) ANO $(\kappa_{+}{=}10, \kappa_{-}{=}{-}1.5)$ significantly outperforms PPO at all temperatures (faithfulness and relevance drive the gap) and is statistically equivalent to SPO. (c,d) The weaker configuration $(\kappa_{+}{=}2, \kappa_{-}{=}{-}0.5)$ preserves both conclusions, indicating that the quality advantage is not an artifact of knife-edge tuning.}
    \label{fig:rubric_forest}
\end{figure}

\begin{figure}[t]
    \centering
    \includegraphics[width=0.62\linewidth]{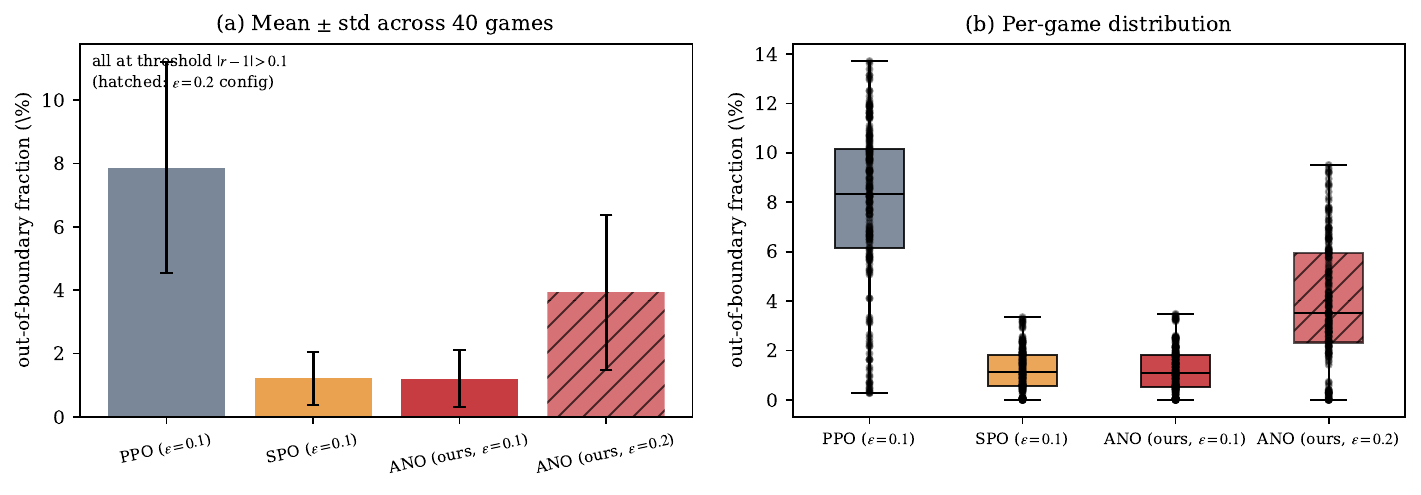}
    \caption{\textbf{Out-of-boundary fraction across all 40 Atari games (5 seeds each).} Mean $\pm$ std (left) and per-game distributions (right) of the time-averaged out-of-boundary fraction at the common threshold $|r-1|>0.1$; companion to Table~\ref{tab:beyond_boundary} and Figure~\ref{fig:clipfrac_asterix}.}
    \label{fig:clipfrac_atari40}
\end{figure}

\begin{figure}[t]
    \centering
    \includegraphics[width=0.40\linewidth]{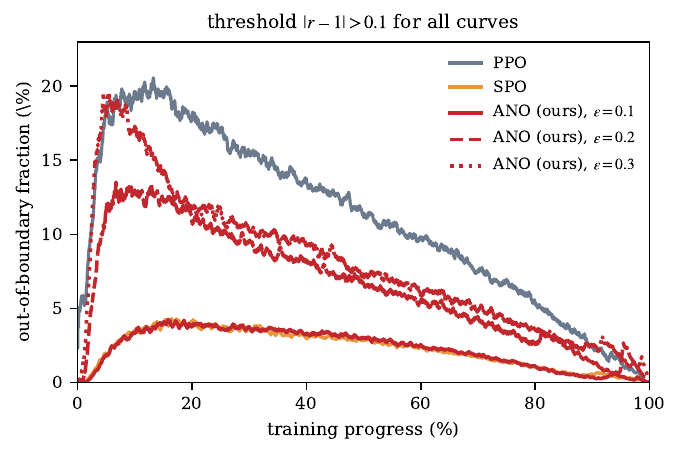}
    \caption{\textbf{Out-of-boundary fraction during training (Atari Asterix, 5 seeds); every curve is logged at the common threshold $|r{-}1|>0.1$} (PPO/SPO at $\epsilon{=}0.1$; ANO's historical logging records the fraction at threshold $0.1$ for all $\epsilon$, verified from run configurations). PPO's dead zone is heavily populated (up to ${\sim}20\%$ of samples receive zero gradient), while ANO's restoring force keeps the fraction $3$--$5\times$ lower; among ANO configurations the fraction grows with $\epsilon$ (a later zero-crossing delays the restoring force), as the gain-field picture predicts.}
    \label{fig:clipfrac_asterix}
\end{figure}

\begin{table}[t]
\centering
\caption{\textbf{Out-of-boundary data fraction on Atari (40 games $\times$ 5 seeds).} Time-averaged fraction of samples with $|r-1|>0.1$ (common threshold; verified from run configurations), mean $\pm$ std across games and seeds. PPO discards gradients on these samples (dead zone); ANO/SPO apply restoring forces to them. ${\dagger}$: $\epsilon{=}0.2$ and $\epsilon{=}0.3$ configurations, fraction measured at the common threshold $0.1$.}
\label{tab:beyond_boundary}
\scriptsize
\begin{tabular}{lc}
\toprule
\textbf{Method} & \textbf{Atari} \\
\midrule
PPO ($\epsilon{=}0.1$) & $7.88 \pm 3.34$ \\
SPO ($\epsilon{=}0.1$) & $1.22 \pm 0.85$ \\
ANO ($\epsilon{=}0.1$, $\kappa_{+}{=}10$, $\kappa_{-}{=}{-}1.5$) & $1.21 \pm 0.90$ \\
ANO ($\epsilon{=}0.2$, $\kappa_{+}{=}10$, $\kappa_{-}{=}{-}1.5$)$^{\dagger}$ & $3.92 \pm 2.45$ \\
ANO ($\epsilon{=}0.3$, $\kappa_{+}{=}10$, $\kappa_{-}{=}{-}1.5$)$^{\dagger}$ & $5.62 \pm 2.93$ \\
\bottomrule
\end{tabular}
\end{table}

\begin{table}[t]
\centering
\caption{\textbf{Learning-rate stress test on MuJoCo (lr $3\times10^{-4} \to 10^{-3}$).} ANO degrades by only $0.9\%$ in IQM ($7\times$ less than SPO), while PPO collapses by $54.5\%$. Notably, the stressed ANO still outperforms PPO and PAPO at their best-tuned learning rates (cf.\ Table~\ref{tab:mujoco_full}).}
\label{tab:robustness}
\scriptsize
\begin{tabular}{lccc}
\toprule
\textbf{Method} & \textbf{IQM} @ $3\times10^{-4}$ & \textbf{IQM} @ $10^{-3}$ & \textbf{Degradation} \\
\midrule
ANO ($\epsilon{=}0.3$, $\kappa_{+}{=}10$, $\kappa_{-}{=}{-}1.5$) & 0.9304 & \textbf{0.9225} & $-0.9\%$ \\
SPO ($\epsilon{=}0.2$) & 0.9249 & 0.8643 & $-6.6\%$ \\
PPO ($\epsilon{=}0.2$) & 0.9185 & 0.4177 & $-54.5\%$ \\
\bottomrule
\end{tabular}
\end{table}

\begin{table}[t]
\centering
\caption{\textbf{RLHF results (TL;DR, Pythia-1B).} Rubric: mean of four $1$--$10$ dimensions by an independent LLM judge, $N{=}300$ test prompts, greedy decoding. KL and entropy: final-window training means. Paired tests at $T{=}0$: ANO vs.\ PPO $\Delta = +0.83$ [$+0.63, +1.02$], $p < 10^{-9}$; ANO vs.\ SPO $\Delta = +0.14$ [$-0.01, +0.29$], TOST-equivalent ($\delta{=}0.5$).}
\label{tab:rlhf_results}
\scriptsize
\begin{tabular}{lccc}
\toprule
\textbf{Method} & \textbf{Rubric} $\uparrow$ & \textbf{KL} $\downarrow$ & \textbf{Entropy} $\uparrow$ \\
\midrule
PPO ($\epsilon{=}0.2$) & $6.76$ & $54.0$ & $21.4$ \\
SPO ($\epsilon{=}0.2$) & $7.43$ & $40.9$ & $23.0$ \\
ANO ($\epsilon{=}0.2$, $\kappa_{+}{=}10$, $\kappa_{-}{=}{-}1.5$) & $\mathbf{7.58}$ & $\mathbf{39.5}$ & $\mathbf{23.7}$ \\
ANO ($\epsilon{=}0.2$, $\kappa_{+}{=}2$, $\kappa_{-}{=}{-}0.5$) & $7.49$ & $42.0$ & $22.6$ \\
\bottomrule
\end{tabular}
\end{table}

\begin{figure}[t]
    \centering
    \includegraphics[width=0.66\linewidth]{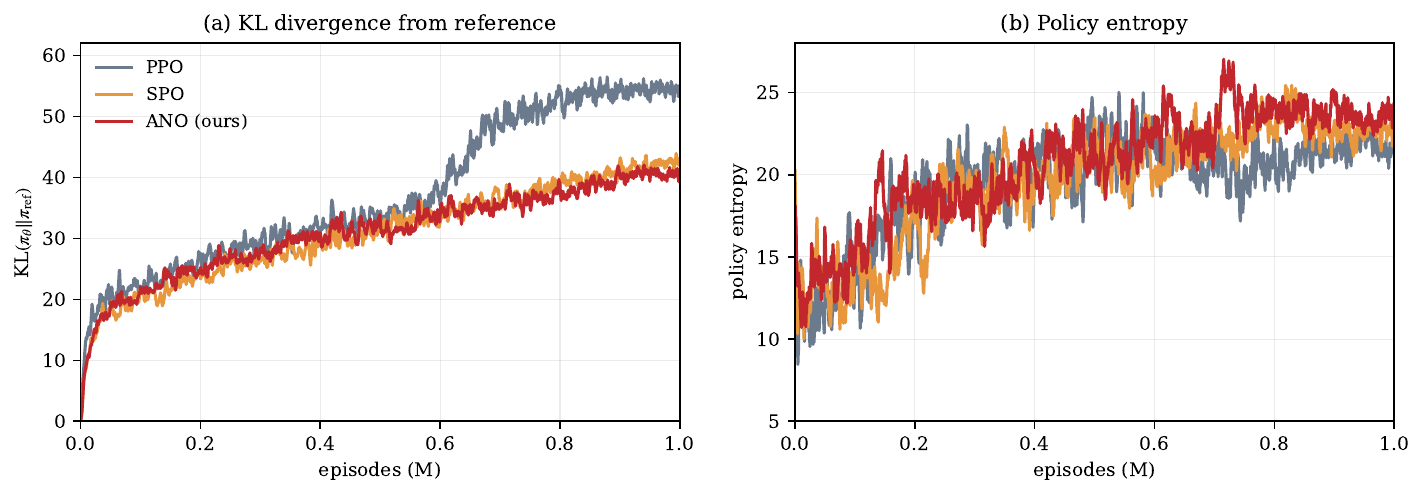}
    \caption{\textbf{RLHF training dynamics (TL;DR, Pythia-1B).} (a) PPO's dead zone lets the policy drift from the reference: the late-phase KL growth rate is ${\sim}2.4\times$ that of ANO ($+4.8$ vs.\ $+2.0$ per $100$k episodes), and the final KL is $54.0$ vs.\ $39.5$. ANO's bounded restoring force keeps the policy closest to the reference, and (b) retains the highest policy entropy.}
    \label{fig:rlhf_dynamics}
\end{figure}

\definecolor{ano}{RGB}{254,136,152}

\clearpage
\allowdisplaybreaks[4]

\section{Proof of Theorem~\ref{thm:dual_bound}}
\label{app:proof_thm_3_1}
\begin{proof}
    From Theorem 1 of \citet{schulman2015trust}, we have
\begin{align}
    \eta(\tilde{\pi}) \geq S(\tilde{\pi}) - \beta \left[D_{\mathrm TV}^{\mathrm max}(\pi,\tilde{\pi})\right]^2.
\end{align}
Suppose $\alpha\in[0,1]$, then we have
\begin{align}
    \eta(\tilde{\pi}) &\geq S(\tilde{\pi}) - \alpha\beta \left[D_{\mathrm TV}^{\mathrm max}(\pi,\tilde{\pi})\right]^2- (1-\alpha)\beta \left[D_{\mathrm TV}^{\mathrm max}(\pi,\tilde{\pi})\right]^2\nonumber \\
    &= S(\tilde{\pi}) - \alpha\beta \left[D_{\mathrm TV}^{\mathrm max}(\tp,\pi)\right]^2- (1-\alpha)\beta \left[D_{\mathrm TV}^{\mathrm max}(\pi,\tilde{\pi})\right]^2\nonumber \\
    &\ge S(\tilde{\pi}) - \frac{\alpha\beta}{2} \left[D_{\mathrm KL}^{\mathrm max}(\tp,\pi)\right]- \frac{(1-\alpha)\beta}{2} \left[D_{\mathrm KL}^{\mathrm max}(\pi,\tilde{\pi})\right]\nonumber\\
    &\quad \text{(By Pinsker's inequality)}\nonumber \\
    &= S(\tilde{\pi}) - \frac{\beta}{2}\Big(\alpha \max_s \mathbb{E}_{a\sim \tp(\cdot|s)}\big[\ln \frac{\tp(a|s)}{\pi(a|s)}\big] + (1-\alpha) \max_s \mathbb{E}_{a\sim \pi(\cdot|s)}\big[\ln \frac{\pi(a|s)}{\tp(a|s)}\big]\Big)\nonumber \\
    &\ge S(\tilde{\pi}) - \frac{\beta}{2}\Big(\alpha \max_{s,a} \big[\ln \frac{\tp(a|s)}{\pi(a|s)}\big] + (1-\alpha) \max_{s,a} \big[\ln \frac{\pi(a|s)}{\tp(a|s)}\big]\Big)
\end{align}
\end{proof}

\section{Proof of Theorem~\ref{thm:base_algo}}
\label{app:proof_thm_3_4}
\begin{proof}
    From Theorem~\ref{thm:dual_bound}, we have
    \begin{align}
        \eta(\tilde{\pi}) &\ge S(\tilde{\pi}) - \frac{\beta}{2}\Big(\alpha \max_{s,a} \big[\ln \frac{\tp(a|s)}{\pi(a|s)}\big] + (1-\alpha) \max_{s,a} \big[\ln \frac{\pi(a|s)}{\tp(a|s)}\big]\Big)\nonumber\\
        &= \eta(\pi) + \frac{1}{1-\gamma}
    \mathbb{E}_{s\sim\rho_{\pi},\;a\sim\pi}
    \Big[ r A_{\pi}(s,a) \Big]- \frac{\beta}{2}\Big(\alpha \max_{s,a} \big[\ln \frac{\tp(a|s)}{\pi(a|s)}\big] + (1-\alpha) \max_{s,a} \big[\ln \frac{\pi(a|s)}{\tp(a|s)}\big]\Big)\nonumber\\
        &= \eta(\pi) + \frac{1}{1-\gamma}
    \mathbb{E}_{s\sim\rho_{\pi},\;a\sim\pi}
    \Big[  \1_{A_\pi(s,a)<0} r A_{\pi}(s,a)+ \1_{A_\pi(s,a)>0} r A_{\pi}(s,a) \Big]\nonumber\\
    &\quad - \frac{\beta}{2}\Big(\alpha \max_{s,a} \big[\ln \frac{\tp(a|s)}{\pi(a|s)}\big] + (1-\alpha) \max_{s,a} \big[\ln \frac{\pi(a|s)}{\tp(a|s)}\big]\Big)\nonumber\\
        &\ge \eta(\pi) + \frac{1}{1-\gamma}
    \mathbb{E}_{s\sim\rho_{\pi},\;a\sim\pi}
    \Big[  \1_{A_\pi(s,a)<0} g(r) A_{\pi}(s,a)+ \1_{A_\pi(s,a)>0} f(r) A_{\pi}(s,a) \Big]\nonumber\\
    &\quad - \frac{\beta}{2}\Big(\alpha \max_{s,a} \big[\ln \frac{\tp(a|s)}{\pi(a|s)}\big] + (1-\alpha) \max_{s,a} \big[\ln \frac{\pi(a|s)}{\tp(a|s)}\big]\Big)\nonumber\\
    &\quad \Big(\text{Because $g(r)\ge r\ge f(r)$}\Big)\nonumber\\
        &\ge \eta(\pi) + \frac{1}{1-\gamma}
    \mathbb{E}_{s\sim\rho_{\pi},\;a\sim\pi}
    \Big[ \min \big( g(r) A_{\pi}(s,a), \, f(r) A_{\pi}(s,a) \big) \Big]\nonumber\\
    &\quad - \frac{\beta}{2}\Big(\alpha \max_{s,a} \big[\ln \frac{\tp(a|s)}{\pi(a|s)}\big] + (1-\alpha) \max_{s,a} \big[\ln \frac{\pi(a|s)}{\tp(a|s)}\big]\Big)\nonumber\\
        &= \underbrace{\eta(\pi) + \frac{1}{1-\gamma}
    M(\tp;f,g) - \frac{\beta}{2}\Big(\alpha \max_{s,a} \big[\ln \frac{\tp(a|s)}{\pi(a|s)}\big] + (1-\alpha) \max_{s,a} \big[\ln \frac{\pi(a|s)}{\tp(a|s)}\big]\Big)}_{M_F(\tp)}.
    \end{align}
According to Corollary~\ref{cor:F0}, we have that $M_F(\pi)=\eta(\pi)$.

It remains to relate the unconstrained maximizer of $M_F$ to the constrained problem. This is the classical exact-penalty fact: for constants $c_1, c_2 > 0$, a maximizer of $H(x) - c_1 \delta_1(x) - c_2 \delta_2(x)$ over any set that contains the feasible region $\{x : \delta_1(x) = \delta_2(x) = 0\}$ is a maximizer of $H(x)$ subject to $\delta_1(x) = \delta_2(x) = 0$, because on the feasible set the penalty terms vanish while off the feasible set the objective is $-\infty$. Applying this to $H(\tp) = \eta(\pi) + \frac{1}{1-\gamma}M(\tp;f,g)$ with $\delta_1(\tp) = \max_{s,a}[\ln \frac{\tp(a|s)}{\pi(a|s)}] - \epsilon_1$ and $\delta_2(\tp) = \max_{s,a}[\ln \frac{\pi(a|s)}{\tp(a|s)}] - \epsilon_2$ (penalized via the extended indicator $\delta_I$, with the constants $\frac{\beta}{2}\alpha$ and $\frac{\beta}{2}(1-\alpha)$ absorbed into it), the unconstrained problem $\max_{\tilde{\pi}} M_F(\tilde{\pi})$ is equivalent to the constrained problem:
\begin{align}
    &\max_\tp M(\tp;f,g)\\
    &\text{s.t. }
    \begin{cases}
        \max_{s,a} \left[\ln \frac{\tp(a|s)}{\pi(a|s)}\right]\le \epsilon_1,\\
        \max_{s,a} \left[\ln \frac{\pi(a|s)}{\tp(a|s)}\right]\le \epsilon_{2}.
    \end{cases}
\end{align}
Obviously, both $\epsilon_l$ and $\epsilon_{u}$ are non-negative. And when $\alpha=0$, $\epsilon_1=\infty$ while $\alpha=1$, $\epsilon_{2}=\infty$.

And the optimization can be rewritten as follows:

\begin{align}
    &\max_\tp M(\tp;f,g)\\
    &\text{s.t. }
    \begin{cases}
        \max_{s,a}\frac{\tp(a|s)}{\pi(a|s)}\le 1+\epsilon_u,\\
        \min_{s,a}\frac{\tp(a|s)}{\pi(a|s)}\ge 1-\epsilon_{l},
    \end{cases}
\end{align}
where $\epsilon_u=e^{\epsilon_{1}}-1$ and $\epsilon_l=1-e^{-\epsilon_{2}}$.

Let $\pi^*$ be the optimal solution to this optimization problem. We have
\begin{align}
\pi^*&=
\argmaxl_{\tp} M(\tp;f,g)-\delta_{I}(\max_r r)-\delta_{I}(\min_r r)\nonumber\\
&=\argmaxl_{\tp} M(\tp;f,g)-\sum_s\sum_a\rho_\pi(s)p_\pi(a|s)\big(\delta_{I}(\max_r r)+\delta_{I}(\min_r r)\big)\nonumber\\
&\quad \Big(\text{Note that $\beta \delta_C(\cdot) \equiv \delta_C(\cdot)$ for any $\beta > 0$}\Big)\nonumber\\
&=\argmaxl_{\tp} \Bigg[\mathbb{E}_{s\sim\rho_{\pi},\;a\sim\pi}\min \big( g(r) A_{\pi}(s,a), \, f(r) A_{\pi}(s,a) \big) -\mathbb{E}_{s\sim\rho_{\pi},\;a\sim\pi}\big(|A_{\pi}(s,a)|\,\delta_{I}(r)\big)\Bigg]\nonumber\\
&=\argmaxl_{\tp} \Bigg[\mathbb{E}_{s\sim\rho_{\pi},\;a\sim\pi}\min \big( g(r) A_{\pi}(s,a), \, f(r) A_{\pi}(s,a) \big) \nonumber\\
&\quad -\mathbb{E}_{s\sim\rho_{\pi},\;a\sim\pi}A_\pi(a|s)\1_{A_\pi(a|s)>0}\big(\delta_{I}(r)\big) +\mathbb{E}_{s\sim\rho_{\pi},\;a\sim\pi}A_\pi(a|s)\1_{A_\pi(a|s)<0}\big(\delta_{I}(r)\big)\Bigg]\nonumber\\
&=\argmaxl_{\tp} \Bigg[\mathbb{E}_{s\sim\rho_{\pi},\;a\sim\pi}\min \Big( \left(g(r)+\delta_{I}(r)\right) A_{\pi}(s,a), \, \left(f(r)-\delta_{I}(r)\right) A_{\pi}(s,a) \Big) \Bigg]\nonumber\\
&=\argmaxl_{\tp} \Bigg[\mathbb{E}_{s\sim\rho_{\pi},\;a\sim\pi}\min \Big( g_\delta(r) A_{\pi}(s,a), \, f_\delta(r) A_{\pi}(s,a) \Big) \Bigg]\nonumber\\
&=\argmaxl_{\tp} M(\tp;f_\delta, g_\delta).
\end{align}
Degenerate case: It is easy to show that when $f(r)=1-\delta_{\{1\}}$ and $g(r)=1+\delta_{\{1\}}$, $\epsilon_u=0$ and $\epsilon_l=0$.
\end{proof}

\section{Symmetric Bounds via $\alpha$-Adjustment (Remark~\ref{rmk:alpha_adjust})}
\label{app:proof_alpha}

\begin{remark}[Symmetric Bounds via $\alpha$-Adjustment]
\label{rmk:alpha_adjust}
Consider a single-state MDP with $\mathcal{A}=\{a_1,a_2,a_3\}$, policy $\pi=[0.2, 0.7, 0.1]$, and advantages $A=[10, -2, -6]$. Let $\beta=8$ and the base shaping functions be linear, $f(r)=g(r)=r$. By setting the convex combination coefficient $\alpha=0.96$ in Theorem~\ref{thm:dual_bound}, the optimal constraints derived for maximizing the objective become perfectly symmetric with $\epsilon_u=\epsilon_l=0.6$.
\end{remark}

\begin{proof}
The optimization problem with the penalty term is:
\begin{equation}
    \max_{\tp} \sum_{a} \tp(a|s) A(a|s) - \frac{\beta\alpha}{2}\max_a \ln\frac{\tp(a)}{\pi(a)} - \frac{\beta(1-\alpha)}{2}\max_a \ln\frac{\pi(a)}{\tp(a)} - \lambda(\sum \tp - 1)
\end{equation}

Assuming the optimal policy $\tp$ hits the upper bound at $a_1$ and the lower bound at $a_3$, while $a_2$ remains in the interior (unconstrained relative to the trust region bounds but satisfying the probability simplex).

The first-order optimality condition (setting gradients to zero) gives:
\begin{align}
    \frac{\partial \mathcal{L}}{\partial \tp_1} &= 10 - \frac{4\alpha}{\tp_1} - \lambda = 0 \\
    \frac{\partial \mathcal{L}}{\partial \tp_2} &= -2 - 0 - \lambda = 0 \implies \mathbf{\lambda = -2} \\
    \frac{\partial \mathcal{L}}{\partial \tp_3} &= -6 + \frac{4(1-\alpha)}{\tp_3} - \lambda = 0
\end{align}

Substituting $\lambda = -2$ into the equations for $a_1$ and $a_3$:
1. For $a_1$ (Upper Bound: $\tp_1 = \pi_1(1+\epsilon)$):
   $$ 10 - \frac{4\alpha}{0.2(1+0.6)} + 2 = 0 \implies 12 = \frac{4\alpha}{0.32} \implies \alpha = 0.96 $$

2. For $a_3$ (Lower Bound: $\tp_3 = \pi_3(1-\epsilon)$):
   $$ -6 + \frac{4(1-\alpha)}{0.1(1-0.6)} + 2 = 0 \implies -4 + \frac{4(1-\alpha)}{0.04} = 0 \implies 1-\alpha = 0.04 \implies \alpha = 0.96 $$

Both conditions consistently yield $\alpha = 0.96$. Thus, the derivation is exact.
\end{proof}

\section{Proofs for the ANO Kernel}
\label{app:proof_ano}

This appendix collects the proofs of the structural results stated in Section~\ref{sec:ano_construction}. Throughout, $x_0 = 1+\epsilon$, $E(x) = e^{a(x-x_0)}$, and the gain field is $G'(x) = \kappa_{+} \nu (1-E)/[(1+E)(E+\nu)]$ with $a>0$, $\nu>0$, $\kappa_{+}>0$.

\subsection{Proof of Proposition~\ref{prop:unique_inflection} (Structural Uniqueness)}
\label{app:proof_inflection}

\begin{proof}
Differentiating $G'(x)$ and using $E'(x) = aE(x)$ gives
\begin{equation}
    G''(x) = \frac{\kappa_{+} \nu a\, E \left(E^2 - 2E - (1+2\nu)\right)}{\big[(1+E)(E+\nu)\big]^2}.
\end{equation}
The denominator is strictly positive and the prefactor $\kappa_{+}\nu aE$ is strictly positive, so
$\operatorname{sign} G'' = \operatorname{sign}\big(E^2 - 2E - (1+2\nu)\big)$.
The quadratic $E \mapsto E^2 - 2E - (1+2\nu)$ has roots $1 \pm \sqrt{2+2\nu} = 1 \pm s$.
Since $\nu > 0$ implies $s = \sqrt{2+2\nu} > \sqrt{2} > 1$, the root $1 - s$ is negative and therefore inadmissible for $E > 0$; the unique admissible root is $E^* = 1 + s$.
As $x \mapsto E(x) = e^{a(x-x_0)}$ is a strictly increasing bijection from $\mathbb{R}$ onto $(0,\infty)$, the equation $E(x) = E^*$ has the unique solution
\begin{equation}
    x_p = x_0 + \frac{\ln(1+s)}{a}.
\end{equation}
The quadratic is negative for $E < E^*$ and positive for $E > E^*$, so $G'$ decreases on $(-\infty, x_p)$ and increases on $(x_p, +\infty)$, whence $x_p$ is the global minimizer of $G'$.
\end{proof}

\subsection{Proof of Lemma~\ref{lem:depth} (Depth Control)}
\label{app:proof_depth}

\begin{proof}
Substituting $E = E^* = 1+s$ into the gain field and using the identity $1 + \nu + s = \tfrac{1}{2}s(s+2)$ (valid since $\nu = \tfrac{1}{2}(s^2-2)$):
\begin{equation}
    G'(x_p) = \frac{\kappa_{+} \nu \big(1-(1+s)\big)}{(2+s)(1+s+\nu)}
    = \frac{-\kappa_{+} \nu s}{(2+s)\cdot \frac{1}{2}s(s+2)}
    = -\,\kappa_{+}\,\frac{s^2-2}{(s+2)^2}
    = -\,\kappa_{+}\, d(\nu),
\end{equation}
where $d(\nu) = (s^2-2)/(s+2)^2$ depends on $\nu$ alone.

For monotonicity, regard $d$ as a function of $s$: $\frac{dd}{ds} = \frac{4(s+1)}{(s+2)^3} > 0$ for $s > \sqrt{2}$. Since $s(\nu) = \sqrt{2+2\nu}$ is smooth and strictly increasing with $ds/d\nu = 1/s$, the chain rule gives
\begin{equation}
    \frac{dd}{d\nu} = \frac{4(s+1)}{s(s+2)^3} > 0.
\end{equation}
The limits follow from $s \to \sqrt{2}$ as $\nu \to 0^+$ (giving $d \to 0$) and $s \to \infty$ as $\nu \to \infty$ (giving $d \to 1$). Strict monotonicity plus continuity and the two limits yield bijectivity of $d: (0,\infty) \to (0,1)$ by the intermediate value theorem, and hence $\kappa_{-} = -\kappa_{+} d(\nu)$ sweeps $(-\kappa_{+}, 0)$ bijectively.

For the inversion, solve $d(s+2)^2 = s^2 - 2$, i.e., $(d-1)s^2 + 4ds + (4d+2) = 0$, whose discriminant simplifies to $8(d+1)$. Retaining the root with $s > \sqrt{2}$:
\begin{equation}
    d = \frac{|\kappa_{-}|}{\kappa_{+}}, \qquad
    s = \frac{2d + \sqrt{2(d+1)}}{1-d}, \qquad
    \nu = \frac{s^2-2}{2}.
\end{equation}
Finally, the pull-to-push ratio satisfies
\begin{equation}
    \frac{|G'(x)|}{\kappa_{+}} = \frac{\nu(E-1)}{(1+E)(E+\nu)} < 1
    \iff 0 < E^2 + E + 2\nu
    \qquad (E > E(x_0) = 1 \text{ on the pull side}),
\end{equation}
which holds for all $E, \nu > 0$; equivalently $\sup_\nu d(\nu) = 1$. Hence the restriction $|\kappa_{-}| < \kappa_{+}$ is intrinsic to the family.
\end{proof}

\subsection{Proof of Proposition~\ref{prop:scale} (Scale Calibration)}
\label{app:proof_scale}

\begin{proof}
Writing $q := E(1) = e^{-a\epsilon} \in (0,1)$, the anchor condition $G'(1) = 1$ reads
\begin{equation}
    \Psi(a) := \frac{\kappa_{+} \nu (1-q)}{(1+q)(q+\nu)} = 1.
\end{equation}
The map $a \mapsto q = e^{-a\epsilon}$ is a strictly decreasing bijection from $(0,\infty)$ onto $(0,1)$, and $q \mapsto \kappa_{+}\nu(1-q)/[(1+q)(q+\nu)]$ is strictly decreasing on $(0,1)$ (numerator decreasing and positive, denominator increasing and positive). The composition $\Psi$ is therefore strictly increasing, with $\Psi(0^+) = 0$ (as $q \to 1$) and $\Psi(+\infty) = \kappa_{+}$ (as $q \to 0$). Since $\kappa_{+} > 1$, the value $1$ lies in the open range $(0, \kappa_{+})$, so a unique root exists by the intermediate value theorem and strict monotonicity.

Clearing denominators turns the anchor condition into a quadratic in $q$:
\begin{equation}
    q^2 + Bq - C = 0, \qquad B := \nu(\kappa_{+}+1)+1 > 0, \quad C := \nu(\kappa_{+}-1) > 0.
\end{equation}
Since $C > 0$, the product of the two roots is $-C < 0$: exactly one root is positive, and it is automatically less than $1$ because substituting $q=1$ gives $1 + B - C = 2 + 2\nu > 0$. Hence
\begin{equation}
    q = \frac{2C}{B + \sqrt{B^2 + 4C}} \in (0,1), \qquad a = -\frac{\ln q}{\epsilon} > 0.
\end{equation}
The conjugate form is preferred over the algebraically equivalent $q = \big({-B + \sqrt{B^2+4C}}\big)/2$: when $C \ll B^2$ the latter subtracts two nearly equal quantities and loses accuracy catastrophically, whereas the conjugate form only ever adds positive terms.
\end{proof}

\subsection{Closed-Form Antiderivative}
\label{app:proof_antiderivative}

Substituting $E = e^{at}$ with $t = x - x_0$, so that $dx = dE/(aE)$, the gain field integrates via the partial-fraction decomposition
\begin{equation}
    \frac{\nu(1-E)}{E(1+E)(E+\nu)}
    = \frac{1}{E} - \frac{2\nu}{(\nu-1)(1+E)} + \frac{\nu+1}{(\nu-1)(E+\nu)}
    \qquad (\nu \neq 1),
\end{equation}
which yields the elementary antiderivative
\begin{equation}
    \Phi(t) = \frac{\kappa_{+}}{a}\left[ at - \frac{2\nu}{\nu-1}\ln(1+E) + \frac{\nu+1}{\nu-1}\ln(E+\nu) \right],
    \qquad E = e^{at}.
\end{equation}
The point $\nu = 1$ is a removable singularity: taking the limit $\nu \to 1$ (e.g., by L'H\^opital's rule applied to the numerator) gives
\begin{equation}
    \Phi(t) = \frac{\kappa_{+}}{a}\left[ at - \ln(1+E) + \frac{2}{1+E} \right],
    \qquad \nu = 1.
\end{equation}
Both branches satisfy $\Phi' = G'$ exactly (differentiate and compare with Eq.~\ref{eq:ano_gain}). Anchoring at $x=1$ to enforce $G(1)=1$ yields $G(x) = 1 + \Phi(x-x_0) - \Phi(1-x_0)$.

\subsection{Asymptotic Behaviour}
\label{app:proof_asymptotics}

Since $G' \to \kappa_{+}$ as $x \to -\infty$ and $G' \to 0$ as $x \to +\infty$, the graph of $G$ possesses an oblique asymptote on the left and a horizontal one on the right:
\begin{align}
    x \to -\infty &: \quad G(x) = \kappa_{+} x + C_{-} + o(1),
    \qquad C_{-} = 1 - \kappa_{+} x_0 + \frac{\kappa_{+}}{a}\big(C_\nu - \Phi(1-x_0)\big), \\
    x \to +\infty &: \quad G(x) \;\longrightarrow\; 1 - \Phi(1-x_0),
\end{align}
where $C_\nu = \frac{\nu+1}{\nu-1}\ln \nu$ for $\nu \neq 1$ and $C_1 = 2$. The right-hand limit is a derived quantity determined by $(\epsilon, \kappa_{+}, \kappa_{-})$, not independently prescribed.

\paragraph{Consistency checks.} Both hold identically by construction:
\begin{equation}
    G'(x_p) = -\,\kappa_{+}\,\frac{s^2-2}{(s+2)^2} = \kappa_{-},
    \qquad
    G'(1) = \frac{\kappa_{+}\nu(1-q)}{(1+q)(q+\nu)} = 1.
\end{equation}
The special value $\nu = 1$ corresponds to $s = 2$ and hence to the single depth $d = \tfrac{1}{8}$, i.e., $\kappa_{-} = -\kappa_{+}/8$.

\section{Proof of Lemma~\ref{lem:enclosure}}
\label{app:proof_enclosure}
\begin{proof}
Let $h = G - \mathrm{id}$. By Proposition~\ref{prop:unique_inflection}, $G'$ decreases on $(-\infty, x_p)$ and increases on $(x_p, \infty)$, with limits $\kappa_{+} > 1$ at $-\infty$ and $0$ at $+\infty$ and minimum $\kappa_{-} < 0$ at $x_p$; since $G'(1) = 1$, we get $G'(x) > 1$ for $x < 1$ and $G'(x) < 1$ for $x > 1$. Hence $h$ increases up to its global maximum $h(1) = 0$ and decreases thereafter, i.e., $G(x) \le x$ everywhere. Then $g(r) = 2 - G(2-r) \ge 2-(2-r) = r$.
\end{proof}

\end{document}